\documentclass[conference]{IEEEtran}
\IEEEoverridecommandlockouts

\usepackage{cite}
\usepackage{amsmath,amssymb,amsfonts}
\usepackage{algorithmic}
\usepackage{graphicx}
\usepackage{textcomp}
\usepackage{xcolor}
\usepackage{tikz}

\usepackage{filecontents}

\usepackage{amsmath,amssymb,amsfonts,amsthm}
\usepackage{algorithmic}
\usepackage{algorithm}
\usepackage{graphicx}
\usepackage{float}
\usepackage{textcomp}
\usepackage{xcolor}
\usepackage{bm}
\usepackage{centernot}

\usepackage[utf8]{inputenc} 
\usepackage[T1]{fontenc}    
\usepackage{hyperref}       
\usepackage{url}            
\usepackage{booktabs}       
\usepackage{amsfonts}       
\usepackage{nicefrac}       
\usepackage{microtype}      
\usepackage{xcolor}         
\usepackage{graphicx}
\usepackage[english]{babel}
\usepackage{bbm}
\usepackage{wrapfig}
\usepackage{enumitem}
\usepackage{bbm}
\usepackage{siunitx}
\usepackage{subcaption}

\theoremstyle{definition}
\newtheorem{definition}{Definition}[section]
\newtheorem{theorem}{Theorem}[section]
\newtheorem{lemma}{Lemma}[section]
\newtheorem{corollary}{Corollary}[section]
\newtheorem{remark}{Remark}[section]
\newtheorem{proposition}{Proposition}[section]
\newtheorem{formulation}{Formulation}[section]

\def\BibTeX{{\rm B\kern-.05em{\sc i\kern-.025em b}\kern-.08em
    T\kern-.1667em\lower.7ex\hbox{E}\kern-.125emX}}

\begin{document}

\title{Certification-Based Differentially Private Learning
}

\author{\IEEEauthorblockN{Mihnea Ghitu}
\IEEEauthorblockA{\textit{Department of Computing} \\
\textit{Imperial College London}\\
London, United Kingdom \\
mihnea.ghitu20@imperial.ac.uk}
\and
\IEEEauthorblockN{Matthew Wicker}
\IEEEauthorblockA{\textit{Department of Computing} \\
\textit{Imperial College London}\\
London, United Kingdom \\
m.wicker@imperial.ac.uk}
}

\maketitle


\begin{abstract}
 
Differential privacy (DP) in machine learning is typically achieved by adding noise to model parameters (\textit{private learning}) or to model outputs (\textit{private prediction}). Recent work uses formal methods, namely abstract interpretation, to provide tighter privacy guarantees, but only for private prediction in classification settings.  In this work, we investigate the use of formal methods as a general tool for tighter privacy analysis. First, we generalize the abstract gradient training (AGT) framework to private prediction in continuous, unbounded regression. Second, by reducing learning in parameterized models to a regression problem over the parameter space, we introduce Abstract Gradient Sampling (AGS), an algorithm that enables reachability-based analysis to provide guarantees for private learning. In both private prediction and private learning, we provide tightened privacy accounting for the AGT framework and a theoretical analysis demonstrating when our smooth sensitivity upper-bounds yield favourable privacy-utility trade-off. In practice, we validate that our regression bounds are tighter than global-sensitivity baselines on regression benchmarks, and, notably, yield the first finite privacy guarantees in settings where global prediction sensitivity is \textit{a priori} unbounded. We also find that under matched conditions, our private learning algorithm can outperform standard private learners.    
 
\end{abstract}
\begin{IEEEkeywords}
Differential Privacy, Formal Verification, Abstract Gradient Training, Private Learning, Private Prediction.
\end{IEEEkeywords}
\section{Introduction}
\looseness=-1
Differential privacy (DP) provides a principled framework for limiting the information revealed about individual data points through the outputs of a computation \cite{dwork2006calibrating, dwork2014algorithmic} and serves as a necessary prerequisite to responsible deployment of learning algorithms in many domains \cite{bommasani2021opportunities}. Privacy is largely achieved through \emph{private learning}, where privacy is enforced during model training by perturbing gradients \cite{abadi2016deep}, objectives \cite{zhang2012functional}, or parameters \cite{chaudhuri2011differentially} despite the tendency of private learning algorithms to substantially reduce model utility \cite{tramerdifferentially}. An alternative that allows performant, non-private models to be used with privacy guarantees is \emph{private prediction} \cite{dwork2018privacy}. In private prediction, privacy guarantees are enforced only at inference time by perturbing the model's outputs according to their sensitivity. This paradigm is particularly appealing in modern machine learning pipelines, where models are often pre-trained, fine-tuned, or deployed in settings where retraining under DP constraints is impractical or undesirable \cite{choquette2021capc}.

While private prediction leaves the model parameters untouched, the noise added to model outputs often erodes the utility of model predictions \cite{van2020trade}. Thus, a central challenge in private prediction is reducing the magnitude of output-privatizing noise by obtaining tighter sensitivity bounds for model outputs \cite{dwork2018privacy, wickercertification}. The use of naive global sensitivity bounds is typically overly conservative, leading to excessive noise and worse utility compared with private learning \cite{van2020trade}. Recent promising work has connected certification-based techniques to smooth sensitivity, yielding improved sensitivity bounds \cite{nissim2007smooth, sosnin2025abstract}. At its core, this offers a way to adapt the noise magnitude to the local geometry of the data and the model, and presents a compelling direction for not only tightening privacy analysis but also leveraging formal methods to inform algorithm design \cite{tsay2026relaxation}.

Unfortunately, the current works that connect formal methods to tightened privacy analysis do so only for private prediction and are restricted to analyzing classification tasks \cite{wickercertification, sosnin2025abstract}. To fully appreciate and unlock the potential of advances in formal methods for improved privacy analysis, we begin by extending prior works to provide tighter privacy guarantees in unbounded and/or continuous regression tasks. To achieve this, we develop a novel certification-based upper-bound on the notion of smooth sensitivity \cite{nissim2007smooth}. We then observe that learning algorithms in parametric models (i.e., algorithms that return parameter vectors) can themselves be viewed as unbounded, continuous regressors. We capitalize on this connection by developing \textit{Abstract Gradient Sampling (AGS)}. Prior work such as~\cite{lipdp} uses Lipschitz bounds to calibrate privacy noise. AGS instead places formal methods at the core of private learning by employing abstract interpretation to soundly bound the sensitivity of each update.

By generalizing prior work that establishes this connection, we conduct a more complete empirical study of where formal methods can improve state-of-the-art privacy guarantees. Theoretically, we provide conditions under which the noise distribution from certification-based algorithms dominates those based on global sensitivity. Notably, our algorithms are, to the best of our knowledge, the first that provide finite \emph{pure differential privacy} ($\varepsilon$-DP) private prediction guarantees when there is not a known, \textit{a priori} upper bound on the global sensitivity. Further, we demonstrate that while pure-DP algorithms are unable to use advanced composition, techniques from formal methods are able to provide composition-like bounds tightening without resorting to approximate DP. 

Empirically, we validate our approach across a wide range of benchmarks. We begin by studying the effect of our private prediction bounds across a toy linear regression setup and the California housing dataset. Across these benchmarks we find that our novel private regression technique, AGT-R, bests global sensitivity private prediction Private Aggregation of Teacher Ensembles (PATE) \cite{papernot2018scalable} by several orders of magnitude and additionally opens up a novel paradigm of tightening privacy analysis by improving formal methods techniques. To benchmark the effectiveness of the AGS algorithm, we study MNIST and the IMDB movie sentiment classification datasets. Similarly, we find that for small values of $\epsilon$ AGS significantly outperforms DP-SGD under matched algorithm conditions (same batch size, learning rate, etc). While the models for which our bounds apply are limited compared to the models analyzable by other private learning algorithms, we hope that the novel connection between  formal methods and privacy in conjunction with the initial results demonstrating dominance in a subset of cases will spur future works in this direction. We summarize our paper's contributions as follows:

\begin{itemize}[leftmargin=*, itemsep=0pt, topsep=0pt]
    \item We extend the Abstract Gradient Training (AGT) framework to regression settings by deriving a novel over-approximation of smooth sensitivity for continuous-valued outputs. 

    \item By connecting learning in parameterized models to unbounded, continuous regression, we present a novel algorithm termed Abstract Gradient Sampling (AGS) which is the first algorithm that uses formal methods as the basis for a differentially private learning algorithm.
    
    \item We theoretically show when the AGT upper bound on smooth sensitivity yields tighter DP compared to both the Laplace and Gaussian mechanisms in regression, and discuss where the AGS algorithm can enhance and improve differentially private stochastic gradient descent (DP-SGD). 


    \item We empirically validate our approach across a hierarchy of model and dataset complexities, from linear regression and deep neural networks to foundation models. In all settings, our method can yield tighter privacy guarantees than existing DP baselines. Our work furthers the emerging connection between formal methods and privacy to the mutual benefit of both sub-fields.
\end{itemize}







\section{Related Work}

\noindent\textbf{Private Prediction for Regression} Differential privacy guarantees are typically achieved by adding noise — scaled to an algorithm's sensitivity — to the algorithm's output \cite{dwork2006calibrating}. When the algorithm is a non-privately trained machine learning model, this can be done by perturbing the model's predictions, a setting termed private prediction \cite{dwork2018privacy}. Global sensitivity, the largest possible change in output between neighbouring datasets \cite{dwork2014algorithmic}, often yields loose bounds; tighter analyses can be obtained by reducing the global sensitivity of a given mechanism \cite{bassily2018model, papernot2016pate}, by exploiting properties of the specific dataset \cite{nissim2007smooth}, or by relaxing to approximate differential privacy \cite{dwork2006our, steinke2015between}. In this work, we focus on differentially private prediction in the regression setting. When the regressor's output is quantised, the PATE framework applies \cite{papernot2016pate, papernot2018scalable}; when it is continuous, one can fall back on standard mechanisms based on global sensitivity, shard-and-aggregate, or data-dependent analysis \cite{dwork2014algorithmic}. To the best of our knowledge, no existing method provides finite privacy guarantees for a non-privately trained regression model whose output is neither quantised nor has an a priori known global sensitivity bound.

\noindent\textbf{Private Learning Mechanisms} Although private prediction offers conceptual advantages over private learning \cite{dwork2018privacy}, its empirical performance lags behind state-of-the-art private learning mechanisms \cite{bassily2018model, van2020trade}. Early work in private learning, like early work in private prediction, added noise to the outputs (and objective functions) of ERM algorithms \cite{chaudhuri2011differentially}. Injecting noise during training instead yielded substantially better utility–privacy trade-offs, most notably via DP-SGD \cite{abadi2016deep}. Subsequent work has steadily tightened the analysis of private learning algorithms \cite{mironov2017renyi, wang2019subsampled, koskela2022individual}, with these advances made widely accessible through programming interfaces such as Opacus \cite{yousefpour2021opacus}. In contrast to these probabilistic refinements, we tighten the analysis of DP-SGD using formal methods.

\noindent\textbf{Combining Formal Methods and Privacy} Formally proving that a machine learning model meets a given specification was largely popularized in response to adversarial robustness \cite{katz2017reluplex, huang2017safety}. However, it has recently been extended to other important notions of trustworthiness \cite{wicker2020probabilistic, benussi2022individual, wicker2022robust, dang2025certifiably} including differential privacy \cite{sosnin2025abstract}. The current use of formal methods to tighten privacy analysis of machine learning models is restricted to private prediction in classification settings. By adopting and extending the general smooth sensitivity-based framework of  \cite{sosnin2025abstract} we push certification-based privacy analysis to private learners.

\section{Preliminaries} \label{sec:prelims}
\looseness=-1
\textbf{Notation. }We denote a general machine learning model as a function $f$ parameterized by $\theta \in \mathbb{R}^p$ which maps from an input space to an output space  $f^{\theta}: \mathcal{X} \rightarrow \mathcal{Y}$ (often with $\mathcal{X} = \mathbb{R}^m$ and $\mathcal{Y} = \mathbb{R}^n$). We consider supervised learning in the \textit{regression} setting with a labeled dataset $D = \{x^{(i)}, y^{(i)}\}_{i=1}^N$ and further assume that $\mathcal{Y}$ is a \textit{continuous} metric space with distance metric $dist(\cdot, \cdot)$. 

\subsection{Privacy}

Differential privacy (DP) is a widely employed formal privacy guarantee, which is deeply rooted in statistical databases \cite{dwork2006calibrating} and has been adopted as the standard for privacy guarantees in machine learning \cite{dwork2014algorithmic,abadi2016deep}. DP ensures that the output of an algorithm applied to two adjacent sets of data is statistically indistinguishable. We formally define DP as follows: 

\begin{definition}[$(\epsilon, \delta)$-DP \cite{dwork2014algorithmic}] \label{def:approx-dp} A randomized mechanism $\mathcal{M}$ is $(\epsilon, \delta)$-differentially private if, for all pairs of adjacent datasets $D, D' \in \mathcal{D}$ and any subset $S \subseteq \text{Supp}(\mathcal{M})$: 
\[
\mathbb{P}\left(\mathcal{M}(D) \in S\right) \leq e^{\epsilon} \mathbb{P}\left(M(D') \in S\right) + \delta.
\]
\end{definition}

$(\epsilon, \delta)$-DP is also known as \textit{approximate} DP, with \textit{pure} or $\epsilon$-DP arising when $\delta = 0$. Within machine learning, there are two distinct approaches to privacy in private prediction and private learning.

\subsubsection{Private Learning} 

A private learning algorithm is one that returns a learned model that itself is differentially private. In machine learning, this is achieved by applying Definition~\ref{def:approx-dp} taking $D$ to be the training dataset, $\mathcal{M}$ to be the learning algorithm, and $\text{Supp}(\mathcal{M})$ to be the space of all possible parameters. To ensure the final parameters, $\theta$, satisfy the definition, Differentially Private Stochastic Gradient Descent (DP-SGD) \cite{abadi2016deep} modifies standard SGD through two core operations at each training step: firstly, it uses a gradient clipping operation to bound the contribution of each training point and adds noise proportional to the clipping parameter to ensure the training algorithms output distribution satisfies Definition~\ref{def:approx-dp} \cite{abadi2016deep}.
Alternative approaches to private learning include DP-ERM \cite{chaudhuri2011differentially}, where output and objective perturbation at training time are shown to satisfy DP, and PATE \cite{papernot2016pate}, which partitions data into disjoint subsets (thus reducing the model's sensitivity through sharding) and employs a teacher-student framework, where the student model distills the teacher's noisy aggregate vote into a publicly trained model, converting private predictions into a private learning procedure. Moreover, substantial advancements have built on top of DP-SGD to enhance privacy analysis including the use of advanced composition theorems \cite{dwork2010boosting} and the moments accountant \cite{abadi2016deep,wang2019subsampled}, among others \cite{bun2016concentrated, pan2024differential}.

\subsubsection{Private Prediction} 
Private learning is not without drawbacks; chiefly, private learning can cause substantial utility degradation which often requires practitioners to retrain their algorithm at different privacy levels potentially compromising their privacy analysis and incurring considerable computational cost \cite{papernot2021hyperparameter, koskela2023practical}. Private prediction represents an alternative to achieving differential privacy in machine learning. The central idea is to first learn a model $f^\theta(x)$ without any privacy protection and to subsequently add noise $\eta$ to the results of each prediction such that Definition~\ref{def:approx-dp} is satisfied \cite{dwork2018privacy}.

To ensure that the released prediction  satisfies $(\epsilon, \delta)$-DP, one typically calibrates the added noise, $\eta$, to the prediction's \textit{sensitivity}.
The most common and general notion of sensitivity is the global $\ell_p$ sensitivity with respect to the cardinality of the symmetric difference $d(D, D') = |D \,\triangle\, D'|$: 
\begin{definition}[Global $\ell_p$ Sensitivity \cite{dwork2014algorithmic}]\label{def:gl1s}A function $f: \mathcal{D} \rightarrow \mathbb{R}^n$ has global $\ell_p$ sensitivity
\[
\Delta_p f = \max_{D,D' \in \mathcal{D}, d(D, D')= 1} \|f(D) - f(D')\|_p.
\]
\end{definition}
Throughout this work we shall refer to the $\ell_1$ sensitivity as $\Delta f$.
With access to the global sensitivity of a prediction, both pure and approximate DP can be satisfied by the following well-known mechanisms: 

\begin{definition}[Laplace $(\epsilon, 0)$-DP Mechanism \cite{dwork2014algorithmic}]
For a function $f: \mathcal{D} \rightarrow \mathbb{R}^n$, the mechanism $\mathcal{M}(D) = f(D) + \eta$ where $\eta \sim \text{Lap}(\Delta f / \epsilon)$ satisfies pure $(\epsilon, 0)$-differential privacy.
\label{def:laplace_mech}
\end{definition}

\begin{definition}[Gaussian Mechanism \cite{dwork2014algorithmic}]
For a function $f: \mathcal{D} \rightarrow \mathbb{R}^n$, the mechanism $\mathcal{M}(D) = f(D) + \eta$ where $\eta \sim \mathcal{N}(0, \sigma^2 \mathbb{I})$ and $\sigma \geq \frac{\Delta_2 f}{\epsilon} \sqrt{2 \ln(1.25/\delta)}$ satisfies $(\epsilon, \delta)$-differential privacy for $\delta \in (0, 1)$.
\label{def:gauss_mech}
\end{definition}

Any mechanism employing the global sensitivity (e.g., Definitions \ref{def:laplace_mech} and \ref{def:gauss_mech}) is data-independent by nature of $\Delta f$ being defined over all possible neighboring datasets $D$ and $D'$. In general, private prediction calibrated to the global sensitivity incurs utility cost exceeding that of private learning \cite{van2020trade}. Though one must be careful to ensure releases remain private, a tighter privacy analysis can be achieved by considering the local sensitivity:
\begin{definition}[Local $\ell_1$ Sensitivity \cite{dwork2014algorithmic}] \label{def:l1_sens}
The local $\ell_1$ sensitivity of a function $f:\mathcal{D} \rightarrow \mathbb{R}^n$ at point $x \in \mathcal{D}$ is defined as $\text{LS}(f, x) = \max_{y: d(x,y) = 1} \|f(x) - f(y)\|_1. $
\end{definition}
As the local sensitivity itself is data-dependent, one cannot directly calibrate noise to the local sensitivity.
%
However, to rectify this, Nissim et. al.  \cite{nissim2007smooth}, proposed a smooth upper bound on this sensitivity with respect to datasets a distance $k$ apart termed the \textit{smooth sensitivity}: 
\begin{definition}[$\beta$-Smooth Sensitivity \cite{nissim2007smooth}] \label{def:beta_ss}
The $\beta$-smooth sensitivity of a function $f:\mathcal{D} \rightarrow \mathbb{R}^n$ at a point $x \in \mathcal{D}$ can be defined as $
    \text{SS}^{\beta}(f, x) = \max_{k \in \mathbb{N}^+} e^{-\beta k} A^k(f,x)
$
where $A^k(f,x) := \max_{y: d(x,y) \leq k} \text{LS}(f, y).$ 
\end{definition}
The major advancement of the smooth sensitivity is the ability to derive associated mechanisms (similar to Definitions \ref{def:laplace_mech} and \ref{def:gauss_mech}) that obtain pure- and approximate-DP. For example, pure-DP can be achieved by adding Cauchy noise calibrated to the smooth sensitivity:
\begin{definition}[Cauchy Mechanism \cite{nissim2007smooth}]
For a function $f: \mathcal{D} \rightarrow \mathbb{R}^n$, the mechanism $\mathcal{M}(D) = f(D) + \eta$ where $\eta \sim \text{Cauchy}(6\,\text{SS}^{\beta}(f,x) / \epsilon)$ and $\forall \beta \leq (\epsilon / 6)$ satisfies $(\epsilon, 0)$-differential privacy.
\label{def:cauchy_mech}
\end{definition}
Similarly, to satisfy approximate-DP, one can use Laplace noise calibrated to the smooth sensitivity:
\begin{definition}[Laplace $(\epsilon, \delta)$-DP Mechanism \cite{nissim2007smooth}]\label{def:laplace_approx_dp_mech} For a function $f: \mathcal{D} \rightarrow \mathbb{R}^n$, the mechanism $\mathcal{M}(D) = f(D) + \eta$, where $\eta \sim \mathrm{Lap}(2\,\mathrm{SS}^\beta(f,x) / \epsilon)$ and $\forall \beta \leq \epsilon / (2 \ln(2/\delta))$ with $\delta \in (0,1)$ satisfies $(\epsilon, \delta)$-differential privacy.
\end{definition}
While Nissim et al. \cite{nissim2007smooth} provide a systematic treatment of admissible noise distributions that yield privacy guarantees based on smooth sensitivity, we focus our attention only on the mechanisms outlined in this section, leaving analysis of our techniques in combination with other mechanisms to future works.

\subsection{Valid Parameter Space Bounds} \label{sec:valid_bounds}

Leveraging formal methods to compute the above smooth sensitivity mechanisms, Sosnin et al.\cite{sosnin2025abstract, wickercertification} cast the local sensitivity of a learning algorithm as a reachability specification that can be verified using abstract interpretation. Their framework computes the parameter envelopes that bound the effect of data modifications during training, thus bypassing the need to optimize over the intractable space of all possible datasets. 
Let $\mathcal{M}(f, \theta_{\text{init}}, D)$ denote a gradient-based algorithm trained on dataset $D$ that returns the final parameters of a model $f$, starting from a fixed initialization $\theta_{\text{init}}$. The parameter envelopes computed by Sosnin et al. \cite{wickercertification, sosnin2025abstract} are termed \textit{valid parameter space bounds}:

\begin{definition}[Valid Parameter-Space Bound \cite{wickercertification}] \label{def:valid_parameter_bounds}
For a nominal dataset $D$ and a distance threshold $k$, a parameter envelope $T_k = [\theta^k_L, \theta^k_U]$ is a valid parameter-space bound if it contains all possible model parameters that result from training on any dataset $\tilde{D}$ formed by up to $k$ additions, removals or substitutions from $D$:
\[ 
\mathcal{M}(f, \theta_{\text{init}}, \tilde{D}) = \tilde{\theta} \in T_k ,\quad \forall \tilde{D} \text{ s.t. } d(D, \tilde{D}) \leq k. 
\]
\end{definition}
\looseness=-1
To compute these intervals using abstract interpretation, Sosnin et al., present \textit{Abstract Gradient Training (AGT)} which leverages bound propagation to construct the envelopes. The algorithm considers batchdx of size $b$ and a per-element gradient clipping threshold $\gamma$. We refer to quantities pertaining to training on the original, unmodified dataset $D$ as \textit{nominal} (i.e., $k=0$). The updates are computed element-wise: AGT aggregates the $b-k$ nominal points' gradients and accounts for the $k$ differing points by assuming they produce gradients equal to $\gamma$ in the most adversarial direction. In this way, it can isolate all reachable parameters to produce final, axis-aligned parameter envelopes $T_k = [\theta_L^k, \theta_U^k]$. 
%
Importantly, this allows the propagation of valid bounds from parameter to output space, acting as a proxy for computing local sensitivity in prediction space. 
\subsubsection*{Shortcomings of Prior Approaches} 
\looseness=-1
The AGT algorithm comes with some downsides. In particular, the bounds provided only apply to private prediction in classification settings. 
Beyond this restrictive setting, the utility of formal-methods-based privacy guarantees remains an open question. Moreover, we highlight that the privacy analysis in \cite{sosnin2025abstract} does not explore the compatibility of AGT and amplification theorems, thus limiting the scale of problems that can be studied with their mechanism. In what follows, we extend the AGT algorithm to regression, discuss the use of amplification and AGT, and provide the first algorithms for using bound propagation for private learning.

\section{Methodology}

\looseness=-1
In this section, we will first introduce the theoretical underpinnings of our novel method, AGT-R, and analyze when is has superiority over global sensitivity.
We will then introduce a novel private learning method, Abstract Gradient Sampling (AGS), which we will similarly prove can be superior to existing approaches.

\subsection{AGT-R} \label{sec:agtr}
We begin by recalling the mechanism presented in Sosnin et al.\cite{sosnin2025abstract, wickercertification} which first uses AGT to compute valid parameter-space bounds via bound propagation for distances (adjacency quantifiers) $k \in \mathcal{K} \subset  \mathbb{N}^+$. The AGT mechanism subsequently employs parameter envelopes $\{T_{k}\}_{k\in \mathcal{K}}$ associated with the adjacency quantifier specified by the subscript. 
Using this notation, we restate their proposed smooth sensitivity bound: 

\begin{theorem}[Sosnin et al. \cite{sosnin2025abstract, wickercertification}]\label{thm:sosnin_binary} Let  $\mathcal{M}(f,\theta_\text{init},D) = \theta$, $T_k = [\theta^k_L,\theta^k_U]$ satisfying Def. \ref{def:valid_parameter_bounds}, and let $f^{\theta}(x)$ denote the prediction of a binary classifier and $\mathbbm{1}(\cdot)$ represent the indicator function. Additionally let $f^{T_k}(x)$ be $1$ when after propagating the envelope through the model $f$, the lower bound of the output envelope satisfies $y_L^k \geq 0.5$, and $0$ otherwise. Then the following is an upper-bound on the $\beta$-smooth sensitivity: 
\begin{align*} 
\overline{\mathrm{SS}^\beta_c}(f_x, D) = \max_{k \in [N]} \left[ \mathbbm{1}\!\left(f^{T_k}(x) \!\neq \!f^{\mathcal{M}(f,\theta_{\mathrm{init}},D)}(x)\right) e^{-\beta k}\right]\!\!.
\end{align*}
\end{theorem}
Albeit sound, a critical observation about this theorem is that it only works in classification settings, where $\forall y,y' \in \mathcal{Y},\, y \neq y'\!: d(y, y') = 1$, and fails otherwise, i.e. in regression settings. We notice that the constraint exists because the indicator function acts as a \textit{sound-bounding function} for the output sensitivity at distance $k$ ($A^k$ in Definition~\ref{def:beta_ss}). To extend this bound to cases where the output space $\mathcal{Y}$ is continuous and unbounded (e.g., $\mathcal{Y} = \mathbb{R}$), we proceed to systematically generalize the sound-bounding function used. 

We first define what it means for a function $\mathfrak{d}(x,k)$  to be sound-bounding for local sensitivity $A^k(f,x)$ by introducing the properties it needs to satisfy: 
\begin{align*}
    (i) & \quad A^k(f,x) \leq \mathfrak{d}(x,k), \: \forall x \in \mathcal{X} , k \in \mathcal{K} \\
    (ii) &  \quad \mathfrak{d}(x,k) \leq \Delta_p f, \forall p \in \mathbb{N}, x \in \mathcal{X} , k \in \mathcal{K}.
\end{align*}

We note that the first property enforces soundness, while the second avoids triviality (i.e., ensures it is lower than the global sensitivity). For $\mathcal{Y} = \mathbb{R}$ one can construct
such a sound-bounding function using only an arbitrary, finite number of envelopes and the mechanism proposed by AGT (further discussed in \S~\ref{ssec:agtr_alg}). For regression, we begin by assuming that we have computed valid parameter-space bounds for all values of $k \in [N]$, and we construct the sound bounding function. Firstly, let $B_k(D)$ be the set of all datasets within radius $k$ of $D$ and allow $y_L^k := \min_{\tilde{D} \in B_k(D)} f^{\tilde{\theta}}(x) \leq \min_{\theta' \in T_k} f^{\theta'}(x)$ and let $y_U^k := \max_{\tilde{D} \in B_k(D)} f^{\tilde{\theta}}(x) \leq \max_{\theta' \in T_k} f^{\theta'}(x)$.  We define the following sound bounding function for regression: $\frak{d}(x,k):= \max(y^k_U - y^{k+1}_L, y_U^{k+1} - y_L^{k})$. For this $\frak{d}(x,k)$, property (i) is satisfied by observing:

\begin{align*}
A^{(k)}(D) &:= \max_{\tilde D \in B_k(D)} LS(f, \tilde D) \\
&\le \max_{\tilde D \in B_k(D),\; D^\star \in B_{k+1}(D)} \|f^{\theta^\star}(x) - f^{\tilde\theta}(x)\|_1 \\
&\le \max_{\tilde y \in [y^k_L,\, y^k_U],\; y^\star \in [y^{k+1}_L,\, y^{k+1}_U]} \|y^\star - \tilde y\|_1
  \\
&= \sum_i \max\left(y^{k+1}_{U,i} - y^{k}_{L,i},\; y^{k}_{U,i} - y^{k+1}_{L,i}\right).
\end{align*}
 
The first inequality follows because $B_1(\tilde D) \subseteq B_{k+1}(D)$ and the second by the soundness of the certified intervals. Property (ii) is satisfied by propagating $k=N$ to get the universe of reachable parameters for which: 
\begin{align*}
 LS_N(f, x) \leq  \max_{\theta' \in T_N} \|f^{\theta}(x) -f^{\theta'}(x)\|_1 \leq GS(f)
\end{align*}
Once a sound-bounding function is established a $\beta$-smooth upper-bound on the smooth sensitivity is given by using the construction of Nissim et al. \cite{nissim2007smooth}. Formally:

\begin{theorem} \label{thm:ss_agtr_ub}
    Given a model $f$, data point $x \in \mathcal{X}$ and $k \in \mathbb{N}$, let $\mathfrak{d}(x,k)$ be an AGT-constructed (as above) sound-bounding function of the local sensitivity $A^k(f,x)$. Then the following is a $\beta$-smooth upper bound on the $\beta$-smooth sensitivity of \cite{nissim2007smooth}: 
    \begin{align*}
    \overline{\mathrm{SS}^\beta}(f_x, D) = \max_{k \in [N]} \left[\mathfrak{d}(x,k)\, e^{-\beta k} \right] 
    \end{align*}
\end{theorem}
\begin{proof}
     The first property of the smooth bound proposed by \cite{nissim2007smooth} follows directly by applying property $(i)$ of a sound-bounding function: $\mathfrak{d}(x,k) \geq A^{k}(f, x)$. The second property, namely $\beta$-smoothness, is a direct consequence of the from the exponential discount factor $e^{-\beta k}$ and the fact that $\mathfrak{d}(D, k) \leq \mathfrak{d}(\tilde{D}, k+1)$ for $d(D, \tilde{D}) = 1$ which is true both when using exact valid parameter-space bounds and any sound orthotope overapproximation.
 \end{proof}

One can observe that by letting $\mathfrak{d}$ be the indicator function, this recovers exactly Theorem \ref{thm:sosnin_binary} of Sosnin et al. \cite{sosnin2025abstract}:

\begin{proposition}
    Let $\mathcal{Y} = \{0,1\}$ and instantiate the sound bounding function with $\mathfrak{d}(x,k) = \mathbbm{1}\!\left(f^{T_k}(x) \!\neq \!f^{\mathcal{M}(f,\theta_{\mathrm{init}},D)}(x)\right)$. Then, Theorem \ref{thm:sosnin_binary} can be recovered and holds: $\overline{\mathrm{SS}^\beta}  = \overline{\mathrm{SS}^\beta_c}.$
\end{proposition}

Thus, sound bounding functions allow us to strictly generalize the analysis provided in prior works by following through with the smooth sensitivity analysis.
%
Using Theorem \ref{thm:ss_agtr_ub} we can achieve $(\epsilon, 0)$ and $(\epsilon, \delta)$-DP guarantees are given through mechanisms in Definition~\ref{def:cauchy_mech} and ~\ref{def:laplace_approx_dp_mech}, respectively.

Lastly, our formulation and privatization mechanism provide a finite bound on smooth sensitivity and noise for unbounded output domains.
\begin{remark}[$\overline{\mathrm{SS}^\beta}$ boundedness for unbounded $\mathcal{Y}$] 

Theorem~\ref{thm:ss_agtr_ub} can provide finite sensitivity bounds even for algorithms with \textit{a priori} infinite global sensitivity  (i.e., $d(\sup \mathcal{Y} - \inf \cal Y) = \infty$) when $\mathfrak{d}(x,N) = \max_{\theta' \in T_N} \|f^{\theta}(x) -f^{\theta'}(x)\|_1 \leq \infty$  in turn yields a finite smooth sensitivity bound $\overline{\mathrm{SS}^\beta}$.
\end{remark}

\subsubsection{Amplification of Generalized AGT Mechanisms}
Before turning to practical computations of generalized sound bounding functions, we remark on the use of amplification with AGT-based mechanisms.

Let \(D=(z_1,\ldots,z_N)\in\mathcal{X}^N\) be a fixed-size dataset
equipped with substitution adjacency, and let
\[
    I \sim \operatorname{Unif}
    \bigl(\{I\subseteq [N]: |I|=m\}\bigr),
    \qquad
    q := \frac{m}{N},
\]
where \(m\leq N\) is public. Denote the resulting size-\(m\)
subsample by \(D_I\). For a fixed public query \(x\), let
\(\mathcal{A}_m(\,\cdot\,;x)\) denote the  AGT
mechanism of Definition~\ref{def:cauchy_mech}, ~\ref{def:gauss_mech} 
instantiated on the randomized dataset of
size \(m\), then we have the following privacy amplification:

\begin{lemma}[Amplification of AGT mechanisms by sampling without replacement \cite{balle2018privacy}]
\label{lem:agtr_wor_amplification}
Suppose that \(\mathcal{A}_m\) is
\((\epsilon_{\mathrm{b}},\delta_{\mathrm{b}})\)-DP
under substitution adjacency on \(\mathcal{X}^m\). Then the subsampled
mechanism $\mathcal{A}^{\mathrm{WOR}}_{N,m}(D;x) := \mathcal{A}_m(D_I;x)$
is $\left(
        \log\!\left(
            1+q\left(e^{\epsilon_{\mathrm{b}}}-1\right)
        \right),
        \;q\delta_{\mathrm{b}}
    \right)$-DP under substitution adjacency on
\(\mathcal{X}^N\).

Equivalently, to obtain a target outer guarantee
\((\epsilon,\delta)\), it suffices to instantiate the base AGT-R
mechanism with
\begin{align}\label{eq:AmplificationEQ}
    \epsilon_{\mathrm{b}}
    =
    \log\!\left(
        1+\frac{e^\epsilon-1}{q}
    \right),
    \qquad
    \delta_{\mathrm{b}}
    =
    \frac{\delta}{q},
\end{align}
where the latter equality assumes \(\delta\leq q\). In the pure-DP
case, \(\delta=\delta_{\mathrm{b}}=0\).
\end{lemma}

Lemma~\ref{lem:agtr_wor_amplification} which is a direct application of \cite[Theorem~9]{balle2018privacy} allows AGT-R to operate on a
single persistent subsample of \(m\) records while providing a privacy
guarantee with respect to the original dataset of size \(N\). 
\begin{algorithm}[H]\footnotesize
\caption{AGT-R (Pure $\epsilon$-DP)}
\label{alg:agt_r}
\begin{algorithmic}[1]
    \STATE \textbf{Input:} Model $f$, initialization $\theta_{\text{init}}$, train dataset
      $D_t$ ($|D_t|=N$), public test input $x$, AGT procedure
      $\mathcal{M}_{\mathrm{AGT}}$, and privacy constants: $\epsilon$, $\delta$, $\beta$, global sensitivity $\Delta f$, and the set of distance radii: $\mathcal{K} \subset \mathbb{N}^+$.
  \STATE \textbf{Output:} An $(\epsilon, 0)$-DP prediction on $x$.
   \vspace{0.3em}
    \STATE $\theta,\;[\theta^k_L, \theta^k_U]_{k\in \mathcal{K}} \leftarrow \forall k_i \in \mathcal{K}, \, \mathcal{M}_{\mathrm{AGT}}(f, \theta_{\text{init}}, D_t, k_i)$
    \vspace{0.3em}
    
    \STATE $\forall i \in [N],\quad \overline{A}[i] \leftarrow \min \left( \Delta f \text{, } \max_{\theta' \in [\theta^N_L, \theta^N_U]}\|f^{\theta'}(x_i) - f^{\theta}(x_i)\|_1 \right)$
    \FOR{\(k\) \textbf{in} \(\text{sort}(\mathcal{K})\)}
        \STATE $\overline{A^k}(f, x) \leftarrow  \max_{\theta' \in [\theta^k_L, \theta^k_U]}\|f^{\theta'}(x_i) - f^{\theta}(x_i)\|_1$ 
        \vspace{0.2em}
        \FORALL{$j \in [N]\, s.t., \, j < k$ }
            \STATE $\overline{A}[i] \leftarrow \min\!\left(\overline{A}[i], \quad \overline{A^k}(f, x)\right)$
        \ENDFOR

    \ENDFOR
    \STATE $\forall i \in [N], \quad B[i] \gets \text{exp}(-\beta i)$
    \STATE $\forall i \in [N], \quad \overline{\mathrm{S}}[i] \gets \overline{A}[i] \cdot B[i]$
    \STATE $\overline{\mathrm{SS}^\beta}\leftarrow \max_k \,\overline{\mathrm{S}}[k]$
      
    \STATE $\hat{y} \leftarrow f^{\theta}(x) + \eta, \quad \eta \sim \mathrm{Cauchy}\!\left( \frac{6\,\overline{\mathrm{SS}^\beta}}{\epsilon}\right)$
    \RETURN $\hat{y}$
\end{algorithmic}
\end{algorithm}
\subsubsection{Tractable Computation via AGT}
 \label{ssec:agtr_alg}

Algorithm~\ref{alg:agt_r} first trains the model via AGT, yielding nominal parameters $\theta$ and parameter envelopes $[\theta^k_L, \theta^k_U]_{k \in \mathcal{K}}$, where $k$ denotes the cardinality of the symmetric difference between neighbouring datasets that parameterizes AGT. 

We observe that in Theorem~\ref{thm:ss_agtr_ub}, $\mathcal{K} := [N]$ implying that one must compute the maximum over all possible symmetric difference radii (with $N$-many AGT runs, for example).
In fact, we show that arbitrary $\mathcal{K} \subset \mathbb{N}$ bounds Theorem~\ref{thm:ss_agtr_ub} as long as $N \in \mathcal{K}$. 
Consider a set $\mathcal{K}$ that has at least one ``gap''  at some value $k$, i.e., $ k\in [N] \wedge k\notin \mathcal{K}$. A sound upper-bound in Theorem~\ref{thm:ss_agtr_ub} must somehow bound the local sensitivity at radius $k$ to achieve a valid smooth-sensitivity bound. We first observe that $\forall k' > k, A^{k'}(f,x) > A^{k}(f,x)$. Thus, if we observe a value $A^{k'}(f,x)$ it is sound for every $k < k'$.
This naturally gives rise to a sound ``back-filling'' procedure:  $\forall  k' \notin \mathcal{K}$ define $k^\star := \min_{k\in\mathcal{K}} k > k'$ and then let $\overline{A^{k'}} = \overline{A^{k^\star}}$. Finally, ensuring that $N \in \mathcal{K}$ allows us to soundly fill all gaps $k' \notin \mathcal{K}$. 
By the same reasoning in Theorem \ref{thm:ss_agtr_ub}, taking the maximum over all $k$ after applying the back-filling procedure obtains a upper-bound on smooth sensitivity (line 13 of Algorithm \ref{alg:agt_r}). 

The choice of $\mathcal{K}$ naturally induces a trade-off: small cardinality sets reduce the number of training runs and forward passes (which grow proportionally with $k$) but may yield looser bounds. Denser sets of values in $\mathcal{K}$  progressively tighten the smooth sensitivity at a greater computational cost. To conclude this analysis, we note that Algorithm \ref{alg:agt_r} covers only the pure $(\epsilon, 0)$-DP case, but few changes are required to achieve approximate $(\epsilon, \delta)$-DP: $\eta_i$ on line $14$ must be sampled from $\text{Lap}(\frac{2\, \overline{\mathrm{SS}^\beta}} {\epsilon})$ with $\beta \leq \frac{\epsilon}{2\ln(2/\delta)}$.

 \subsubsection{Improving Utility of Private Prediction} \label{sssec:superiority_thms}
 
In this section, we establish the conditions under which AGT-R can be (probabilistically) guaranteed to outperform global sensitivity-based privatization in regression settings. We analyze equivalent approaches in pure- and approximate-DP, providing closed-form privacy-utility trade-offs. Proofs of theorems are deferred to Appendix \ref{app:proof_agtr_cauchy_dom} and \ref{app:proof_agtr_lap_dom}.

\subsubsection*{Pure $(\epsilon,0)$-DP}
\looseness=-1
The noise distributions that guarantee $(\epsilon,0)$-DP are $\text{Lap}(0;\Delta f / \epsilon)$ and $\text{Cauchy}(0; 6\,\text{SS}^{\beta}(f, x) / \epsilon)$ for global and smooth sensitivity, respectively. We analyze and prove the critical conditions that need to be satisfied for smooth sensitivity privatization to outperform global sensitivity privatization at a \textit{fixed} privacy level $\epsilon$. Importantly, we make this systematic by defining a flexible cost fraction $c$, which represents the \textit{exact} utility improvement multiplier our method offers, provided the condition below is satisfied.

\begin{theorem} \label{thm:cauchy_lap_dom}
    Consider a regressor $f$ trained on dataset $D$ performing private prediction at point $x$ and a real constant $c \in (1,\infty)$. At a failure probability $\alpha$ (or, equivalently, at confidence level $1 - \alpha$), Cauchy noise with parameter $\gamma=6 \text{SS}^{\beta}(f, x) / \epsilon$ and $\beta < \epsilon/6$ provides $c$-tighter error bounds than Laplace noise with parameter $b = \Delta f /\epsilon$ whenever 
    \[
    \text{SS}^\beta(f,x) < \frac{\Delta f}{6c} \left[\frac{\ln\left(\frac{1}{\alpha}\right)}{\tan \left(\frac{\pi(1-\alpha)}{2}\right)}\right].
    \]
\end{theorem}

\begin{figure*}[t]
    \centering
    \includegraphics[width=\textwidth]{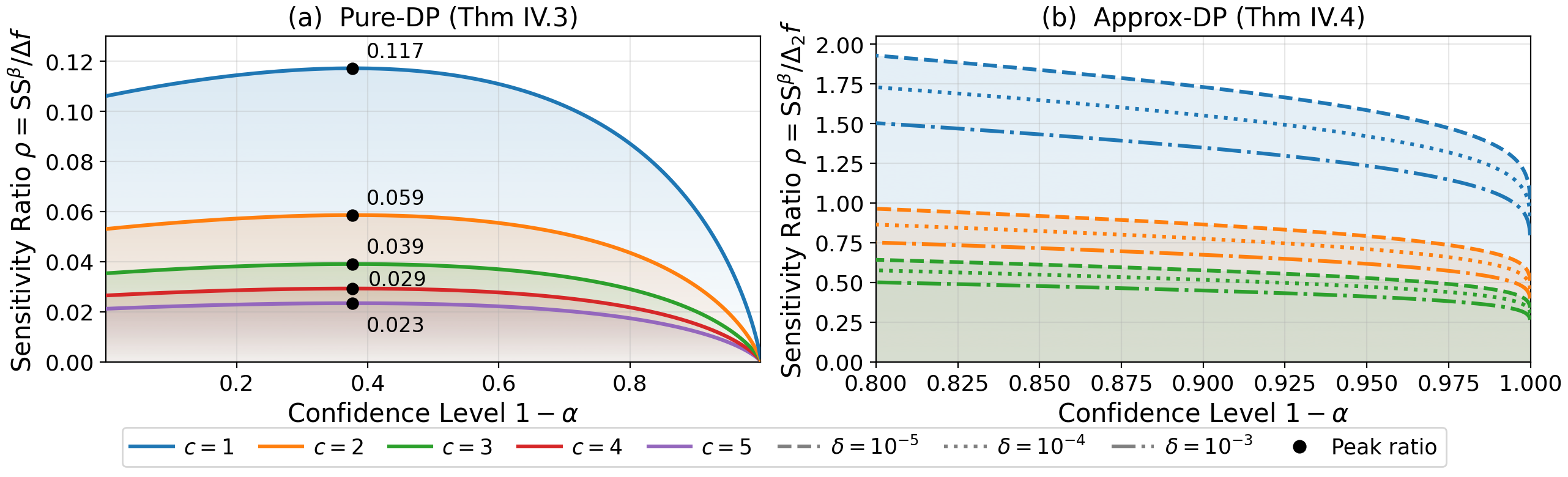}
\caption{\textbf{Feasibility of smooth-sensitivity privatization.} Each curve
gives the largest sensitivity ratio $\rho=\mathrm{SS}^\beta/\Delta f$ for which
our smooth-sensitivity mechanism attains a $c$-times \emph{tighter} error bound
than the global-sensitivity baseline at confidence $1-\alpha$; the region below
each curve is feasible. \textbf{(a)}~Pure-DP (Thm.~\ref{thm:cauchy_lap_dom}),
Cauchy vs.\ Laplace; black dots mark the peak ratio. \textbf{(b)}~Approx-DP
(Thm.~\ref{thm:laplace_gauss_dom}), Laplace vs.\ Gaussian, for $1-\alpha>0.8$
and measured against $\Delta_2 f$. 
}
    \label{fig:cond_table}
\end{figure*}
\noindent \textbf{Analysis of Utility Trade-offs:} The condition in Thm.~\ref{thm:cauchy_lap_dom} is a probabilistic certificate of utility improvement, and we make two observations about it. Firstly, since an exact expression for $\text{SS}^\beta$ is not available, the left-hand side is in practice replaced by the upper bound $\overline{\text{SS}^\beta}$ computed by AGT, which makes the condition sufficient but conservative. Secondly, as $\alpha \rightarrow 0$ we have $\tan(\pi(1-\alpha)/2) \sim 2/(\pi\alpha)$, so the bracketed term decays as $\tfrac{\pi}{2}\alpha\ln(1/\alpha) \rightarrow 0$: the critical sensitivity vanishes faster than the $\ln(1/\alpha)$ grows and thus no deterministic guarantee is attainable. The condition therefore fails when the AGT bounds are unacceptably large, when the enforced improvement factor $c$ is unrealistic, or when the desired guarantee approaches determinism. Although we cannot control the smooth sensitivity directly, it is possible to simulate over $\alpha$ and $c$ and use AGT-R only when the ratio between $\overline{\mathrm{SS}^\beta}$ and $\Delta f$ falls below the resulting critical value. We present this simulation in the left column of Fig.~\ref{fig:cond_table} and analyze it in \S\ref{sec:exp}.

\subsubsection*{Approximate $(\epsilon, \delta)$-DP}
\looseness=-1
We now turn to the Gaussian mechanism, which, although it offers lighter tails via $(\epsilon,\delta)$-DP, requires $\delta \ll 1/N$, making its noise scale prohibitive in limited-sample regimes. Similarly to before, we derive and prove a sufficient condition on the smooth sensitivity under which our smooth sensitivity-based Laplace mechanism (approximate $(\epsilon,\delta)$-DP) yields tighter error bounds than the Gaussian mechanism.

\begin{theorem} \label{thm:laplace_gauss_dom}
    Consider the regressor $f$, dataset $D$, point $x$,  confidence level $1 -\alpha$, and real constant $c \in (1,\infty)$. Additionally, consider a fixed failure probability $\delta \in (0,1)$ of approximate-DP. Laplace noise with scale parameter $b = 2\,\text{SS}^\beta(f,x)/\epsilon$ and $\beta < (\epsilon / (2 \ln(2/\delta)))$ provides $c$-tighter error bounds than Gaussian noise with parameter $\sigma = \frac{\Delta_2f \sqrt{2\cdot\ln(1.25/\delta)}}{\epsilon}$ whenever: 
    \[
    \text{SS}^\beta(f,x) <  \frac{\Delta_2 f}{2c}\left[\frac{\sqrt{2\ln\left(\frac{1.25}{\delta}\right)} \cdot \Phi^{-1}\left(1 - \frac{\alpha}{2}\right)}{\ln \left(\frac{1}{\alpha}\right)}\right],
    \]
    where $\Phi(\cdot)$ denotes the standard normal CDF and $\delta$ is the allowed privacy leakage in approximate DP.
\end{theorem}       
\noindent \textbf{Analysis of Utility Trade-offs:}  The condition in Thm. \ref{thm:laplace_gauss_dom} is again a probabilistic certificate, $\Delta_2 f$, $c$, and $\overline{\text{SS}^\beta}$ playing the same role as before. 
The confidence level, however, behaves differently. Firstly, both quantiles grow as $\alpha \rightarrow 0$, since $\Phi^{-1}(1 - \alpha/2) \sim \sqrt{2\ln(1/\alpha)}$, so the bracketed factor decays only as $\sqrt{2/\ln(1/\alpha)}$ rather than collapsing polynomially.
This reflects the tails involved, as Laplace and Gaussian noise are both light-tailed, whereas Cauchy noise is not. Secondly, the factor $\sqrt{2\ln(1.25/\delta)}$ acts in our favour, since a smaller leakage $\delta$ forces the Gaussian mechanism to inject more noise. In summary, this condition is markedly more permissive, and remains satisfiable at confidence levels for which Thm. \ref{thm:cauchy_lap_dom} already fails. We report the critical ratio $\mathrm{SS}^\beta / \Delta_2 f$ for various values of $\alpha$ and $c$ at fixed $\delta$ in the right column of Figure ~\ref{fig:cond_table}.

\section{Abstract Gradient Sampling}

The previous section established AGT-R as a mechanism for private \textit{prediction}: given a trained model, we add noise calibrated to smooth sensitivity in the \textit{output} space to privatize individual predictions. We now turn to what is arguably the more fundamental problem in the ML privacy literature: private \textit{learning}---i.e., releasing a model whose parameters themselves satisfy differential privacy. The gold standard for private learning is widely considered to be DP-SGD \cite{abadi2016deep}, with improvements such as R\'{e}nyi DP \cite{mironov2017renyi} and the moments accountant \cite{wang2019subsampled, mironov2019sgm}. We show that the smooth sensitivity framework developed for AGT-R extends naturally to the parameter space, yielding a novel approach to private learning. The central observation is that where the learning algorithm is viewed as a multi-dimensional regression over parameters, $\theta$, the valid parameter space bounds yield exactly the local sensitivities used to calibrate noise in AGT-R. In this section, we begin by directly applying AGT-R to the output of the learning algorithm (\S\ref{ssec:post_train_priv}); we then systematically generalize this approach into an algorithm we call Abstract Gradient Sampling (\S\ref{ssec:ags}).

\subsection{Post-Training Privatization} \label{ssec:post_train_priv}


Where the learning algorithm returns a real-valued vector, $\theta = \mathcal{M}(f, \theta_{\text{init}}, D)$, the $\ell_1$ sensitivity of that vector is given by the valid parameter space bound with $k=1$ i.e., $T_1$ i.e., Definition~\ref{def:l1_sens}. Moreover, the parameter envelopes produced by AGT for any $k\geq1$ maintain the guarantee for that for any neighboring dataset $\tilde{D}$ with $|\tilde{D} \, \triangle\, D| \leq k$, the parameters obtained by training on $\tilde{D}$ reside within the bounds. Formally, we have that if $\tilde{\theta} = \mathcal{M}(f, \theta_{\text{init}}, \tilde{D})$, then by construction, we have that $\tilde{\theta} \in [\theta_L^k, \theta_U^k]$. As in the (private prediction) regression case, the local sensitivity at $\tilde{D}$ compares $\tilde{\theta}$ with the parameters of a neighbor of $\tilde{D}$, which lie in $[\theta_L^{k+1}, \theta_U^{k+1}]$, so the sound-bounding function construction of \S\ref{sec:agtr} gives: 
\[
\overline{A^k}(f_\theta, D) = \sum_i \max\left(\theta^{k+1}_{U,i} - \theta^{k}_{L,i},\; \theta^{k}_{U,i} - \theta^{k+1}_{L,i}\right).
\]
It is now clear that $\overline{A^k}(f_\theta, D)   \geq A^k(f_\theta, D)$.
Thus, applying the same construction as in Alg.~\ref{alg:agt_r}, yields a $\beta$-smooth upper bound on the smooth sensitivity:
\[
\overline{\mathrm{SS}^\beta}(f_\theta, D) = \max_{k} \, e^{-\beta k} \cdot \overline{A^k}(f_\theta, D) \geq \mathrm{SS}^\beta(f_\theta, D).
\]
Finally, the privatization mechanism follows naturally: indeed, calibrating Cauchy noise to $\overline{\mathrm{SS}^\beta}(f_\theta, D)$ ensures the released parameters themselves are private, which exactly matches the guarantee of DP-SGD:

\begin{corollary}[Single-Release AGS]\label{thm:1_step_ags}
Let $f$ be a model with parameters $\theta \in \Theta \subseteq \mathbb{R}^{p}$, let
$\theta_{\mathrm{init}}$ be a random initialization,
and let $\theta^{n_s} = \mathcal{M}(f, \theta_{\mathrm{init}}, D)$ be the parameters
obtained after $n_s$ training steps. Let
$u(\beta) := \overline{\mathrm{SS}^\beta}(f_{\theta^{n_s}}, D)$ denote the $\beta$-smooth sensitivity. The
release $\theta^{n_s}_{\mathrm{private}}
  \;=\; \theta^{n_s} \;+\; \eta_{u(\beta)}$, where $\eta_{u(\beta)} \in \mathbb{R}^p$ has independent coordinates, each with density:
\begin{align*}
  \textbf{\emph{(C)}} \;\; \eta_{u(\beta), j} &\sim \mathrm{Cauchy}\!\left(6\,u/\epsilon\right),
    &&\beta \leq \epsilon/(6p), \\
  \textbf{\emph{(L)}} \;\; \eta_{u(\beta), j} &\sim \mathrm{Lap}\!\left(2\,u/\epsilon\right),
    &&\beta \leq \epsilon/\big(2p\,(1 + 2\ln(2p/\delta))\big),
\end{align*}
satisfies pure-DP (for C) and approx. DP (for L).
\end{corollary}
\begin{proof}
Both noise densities are admissible in the sense of \cite[Def.~2.4]{nissim2007smooth}, so the claim is \cite[Lemma~2.5]{nissim2007smooth} with $S = u(\beta)$ and $\alpha = \epsilon/6$, resp.\ $\alpha = \epsilon/2$; the admissibility is proved in Appendix~\ref{app:admisibillity_proof}.
\end{proof}

\subsubsection*{Advanced Composition for Repeated Release} Where we release the parameters $n$ times along one trajectory, we can apply advanced composition in order to get tighter privacy guarantees than basic composition. Formally:

\begin{corollary}[Advanced Composition for Repeated Release]
\label{cor:ags_coordinate_advanced_composition}
Suppose that each of the $n$ releases along the trajectory of Algorithm~\ref{alg:ags} is obtained by Corollary~\ref{thm:1_step_ags} at budget $(\epsilon_0,\delta_0)$.
Then, for any $\delta_{\mathrm{ac}}\in(0,1)$, their joint release is
\[
\left(
    \epsilon_0\sqrt{2n\ln(1/\delta_{\mathrm{ac}})}
    +n\epsilon_0\!\left(e^{\epsilon_0}-1\right),
    \;
    n\delta_0+\delta_{\mathrm{ac}}
\right)\text{-DP}.
\]
\end{corollary}
\begin{proof}
Each release is $(\epsilon_0,\delta_0)$-DP by Corollary~\ref{thm:1_step_ags} and the releases are adaptively composed along one trajectory, so the bound is the advanced composition theorem of \cite[Thm.~III.3]{dwork2010boosting}.
\end{proof}

\subsection{Abstract Gradient Sampling (AGS)} \label{ssec:ags}

Using AGT-R for post-training privatization bears substantial similarities with existing private training algorithms: gradients are clipped to enforce a bounded sensitivity and carefully calibrated noise is added to privatize the model parameter. In popular private learning algorithms, however, noise addition is done at each step of the learning process. In this section, we generalize post-training privatization to allow the application of AGT-R privatization at intermediate steps of the learning algorithm. We term this generalization \textit{Abstract Gradient Sampling}. 

To make this generalization we define the notion of intermediate privatization through \textit{``sampling''} at step $i \in \{1,\dots,n_s\}$: exactly as in Corollary~\ref{thm:1_step_ags} but with $u(\beta, i)$ defining an upper-bound on the $\beta$-smooth sensitivity at step $i$ we can use the single release mechanism: $\theta_{\text{private}}^i = \theta^i + \eta_{u(\beta, i)}$
%
While $\theta_{\text{private}}^i$ can be released with an accompanying privacy guarantee, we are often only interested in the final learning parameter, therefore, we must establish how to carry out our private learning algorithm from an intermediate point. In what follows we first demonstrate how multiple releases affects the AGT semantics and how to leverage tight privacy accounting for multiple releases. 

\subsubsection{Multiple Release and AGT Semantics}

The AGT algorithm accounts for dataset sensitivity through the use of abstract interpretation and formally verified parameter envelopes. Once a parameter has been released with DP guarantees, however, the DP mechanism accounts for the dataset sensitivity, and intuitively the parameter envelopes are no longer necessary. We make this formal with Lemma~\ref{lem:collapse}:

\begin{lemma}[Envelope Collapse]\label{lem:collapse}
Sampling collapses certified envelopes $\forall k$ onto the released point: $[\theta^{k,i}_L,\, \theta^{k,i}_U] \;=\; \{\theta^i_{\mathrm{private}}\}.$
\end{lemma}

\begin{proof}
To safely collapse envelopes one must prove two properties (i) the privacy guarantee is accounted for and (ii) the AGT semantics are preserved. Satisfaction of (i) follows directly from Thm.~\ref{thm:1_step_ags} which gives the DP mechanism. Satisfaction of (ii) is done by observing that if $\forall k, \theta^{k}_L = \theta^{k}_U$, then $\forall k \max_k \frak{\delta}(x, k) = 0$ and finally $\forall \beta, \overline{\mathrm{SS}^\beta} = 0$ implying that the AGT release deterministically reports $\theta^{i}_{\text{private}}$. Thus collapsing all envelopes preserves the privacy guarantee and maintains valid AGT semantics. 
\end{proof}

\subsubsection{Tighter Privacy Accounting for AGS}

The final privacy guarantee for AGS must cover all of its parameter releases, which can be done with composition: 

\begin{theorem}[AGS Sequential Composition]\label{thm:AGS-comp}
Let $S = \{i_1 < \dots < i_m\} \subseteq \{1,\dots,n_s\}$ with $i_m = n_s$ be the sampling steps, partitioning training into \textit{windows} $(i_{j-1}, i_j]$ (with $i_0 = 0$), each certified by $\mathcal{M}_{\mathrm{AGT}}$ from the release preceding it. If the release at step $i_j$ is $(\epsilon_j, \delta_j)$-DP, then $\theta^{n_s}_{\mathrm{private}}$ is released with guarantee  $\left(\textstyle\sum_{j=1}^{m}\epsilon_j,\;\; \sum_{j=1}^{m}\delta_j\right)\text{-DP}.$
\end{theorem}
As with prior sections, we highlight that sequential composition is substantially looser the state-of-the-art numerical composition \cite{gopi2021numerical}. We can leverage these results by observing that for every AGS mechanism (Corollary~\ref{thm:1_step_ags}), $\mathcal{M}$, we can let
\[
    \delta_{\mathcal{M}}(\epsilon)
    :=
    \inf\bigl\{
        \delta :
        \mathcal{M}\text{ is }(\epsilon,\delta)\text{-DP}
    \bigr\}
\]
denote its privacy curve which can be subsequently optimized by numerical accounting techniques which can substantially improve the privacy guarantee \cite{gopi2021numerical}.  

\subsubsection{Statement of the AGS Algorithm}

The complete algorithm we propose, \textbf{Abstract Gradient Sampling (AGS)}, can be found in Algorithm \ref{alg:ags}. AGS receives a user-defined set of steps at which sampling should occur, along with the desired privacy budget for the windows. Then, for each window, AGT is performed to obtain parameter envelopes and compute the upper bound on smooth sensitivity, which is subsequently leveraged to perform the noise addition and privatize the window's parameters, as per Theorem \ref{thm:1_step_ags}. Since the privacy cost has been paid by sampling, the parameter bounds are reset (Lemma \ref{lem:collapse}) and the budget is accounted for through composition (Theorem \ref{thm:AGS-comp}). Finally, the algorithm returns a private set of parameters, along with a total privacy guarantee.


\subsection{Analysis of AGS} \label{ssec:ags_analysis}

The power of AGS lies in the ability to control the moment of sampling (i.e., privatization). This is a key advantage: indeed, due to this particularity of our algorithm, it is possible to design a strategy that only releases parameters when the smooth-sensitivity-calibrated noise has a lower magnitude than the noise DP-SGD would accumulate over the same steps. We investigate this strategy theoretically and provide closed-form expressions for the conditions that need to be satisfied for AGS to be preferable to DP-SGD. 
\begin{algorithm}[H]\small
\caption{AGS}
\label{alg:ags}
\begin{algorithmic}[1]
    \STATE \textbf{Input:} Dataset $D$, function $f$, initialization $\theta_{\mathrm{init}}$, sampling steps set $S \subseteq \{1,\dots,n_s-1\} \cup \{n_s\}$, step privacy budgets $\epsilon_i, \delta_i,\, \forall i \in S$.
    \STATE $\epsilon \leftarrow 0, \; \delta \leftarrow 0$
    \STATE $\theta^0_L = \theta_{\mathrm{init}} = \theta^0_U$
    \vspace{0.2em}
    \STATE $\theta^{\mathrm{curr}}_{\mathrm{private}} = \theta_{\mathrm{init}}$ \COMMENT{Window's private parameter start}
    \vspace{0.3em}
    \FOR{$i$ \textbf{in} $S$}
        \STATE $[\theta_L^i,\theta_U^i],\theta^i \leftarrow \mathcal{M}_{\mathrm{AGT}}(f, \theta^{\mathrm{curr}}_{\mathrm{private}}, D, \cdot)$
        \vspace{0.2em}
        \STATE $\theta^{\mathrm{curr}}_{\mathrm{private}} \leftarrow \theta^i + \mathrm{AGS}(\theta_L^i,\theta_U^i)$
        \vspace{0.2em}
        \STATE \COMMENT{Paid privacy cost - reset interval}
        \vspace{0.2em}
        \STATE $\theta^i_{U} = \theta^i_{L} = \theta^{\mathrm{curr}}_{\mathrm{private}}$ 
        \vspace{0.2em}
        \STATE $\epsilon \leftarrow \epsilon + \epsilon_i, \; \delta \leftarrow \delta + \delta_i$
    \ENDFOR
    \vspace{0.3em}
    \RETURN $\theta^{\mathrm{curr}}_{\mathrm{private}}, \epsilon, \delta$
\end{algorithmic}
\end{algorithm}
\subsubsection*{Utility Analysis - Pure DP}
The analysis of the incurred privacy cost is straightforward via the scale parameter and standard composition for both Laplace DP-SGD and AGS. However, utility is considerably difficult to characterize: noise addition in parameter space affects the per-step quality parameter estimate non-linearly and additionally, gradient descent counteracts this effect. Thus, the utility analyses of Theorems \ref{thm:cauchy_lap_dom} and \ref{thm:laplace_gauss_dom} are not applicable. While our prior bounds introduced no additional assumptions on top of those made by Sosnin et. al. \cite{sosnin2025abstract}, in what follows we will need to reason about training dynamics, which require introducing additional assumptions. For $(\epsilon,0)$-DP, we present below a condition for AGS to outperform Laplace DP-SGD. The theorem's proof is deferred to Appendix \ref{app:proof_ags_pure}.

\begin{theorem}\label{thm:1_step_ags_cauchy}
Let the loss function $\mathcal{L}(\bm{\theta}, D)$ be $1$-dimensional, $\mu$-strongly convex, and $L$-smooth with a locally constant Hessian across the $s_p$-step window. Consider the DP-SGD parameter update over $s_p$ steps with a learning rate $\eta \le 1/L$, where the gradient is perturbed by i.i.d. noise $Z_s \sim \mathrm{Lap}((s_p\Delta f) / \epsilon)$ with mean zero (so that the final per-window privacy guarantee is ($\epsilon,0)$-DP by basic composition). For a user-defined confidence level $\alpha \in (0, 1)$, a single AGS update using Cauchy noise with scale $\gamma = 6\,\mathrm{SS}^\beta(f)/\epsilon$ and $\beta < \epsilon/6$ yields a strictly tighter \textit{parameter space} error bound than Laplace DP-SGD with probability $1-\alpha$ whenever:
\begin{align*}
    \mathrm{SS}^\beta(f_{\theta},D) & < \frac{s_p \Delta f}{3 \tan\left(\frac{\pi}{2}(1-\alpha)\right)} \times \\  \times \max &\left( \sqrt{\frac{\eta \ln{\frac{\alpha}{2}}}{c\,\mu(\eta\mu - 2)}}, \; \frac{-\eta \ln{\frac{\alpha}{2}}}{c} \right)\!,
\end{align*}
where $c > 0$ is an absolute constant.
\end{theorem}

\subsubsection*{Utility Analysis - Approximate DP}

We now take the \textit{exact} same approach in the approximate DP case, namely Gaussian DP-SGD versus AGS with Laplace noise calibrated to the smooth sensitivity. Proof is deferred to Appendix \ref{app:proof_ags_approx}.


\begin{figure*}[t]
  \centering
  
  \begin{subfigure}[b]{\textwidth}
    \centering
    \includegraphics[width=\textwidth]{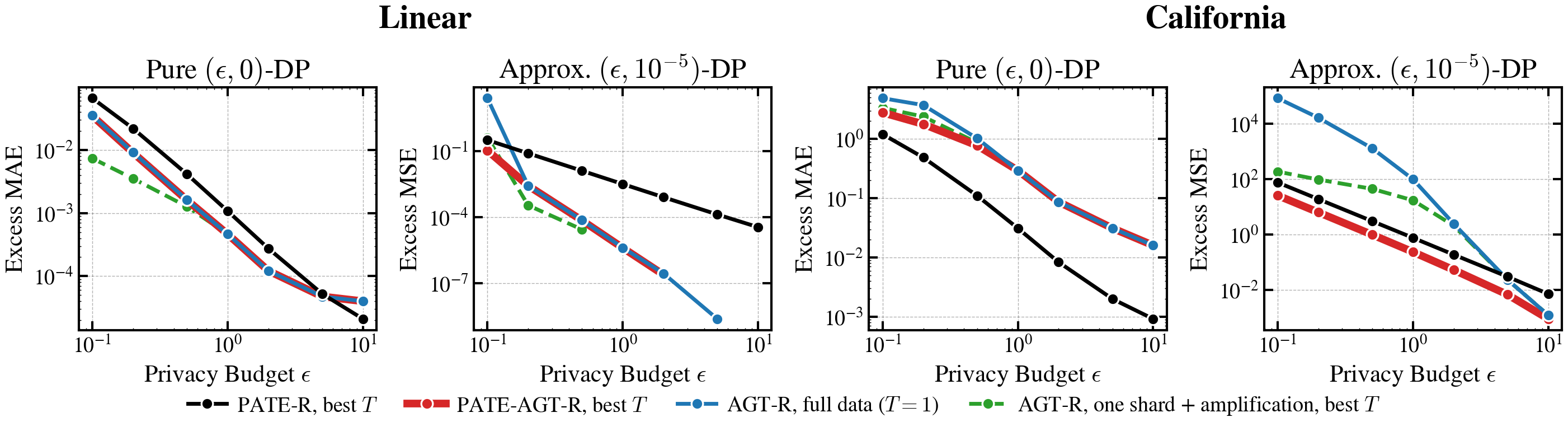}
    \caption{PATE comparison. Error above the non-private one for PATE-R and PATE-AGT-R at their best $T$, AGT-R on the full data, and AGT-R on one shard with the amplified budget.}
    \label{fig:agtr_pate}
  \end{subfigure}
  \begin{subfigure}[b]{\textwidth}
    \centering
    \includegraphics[width=\textwidth]{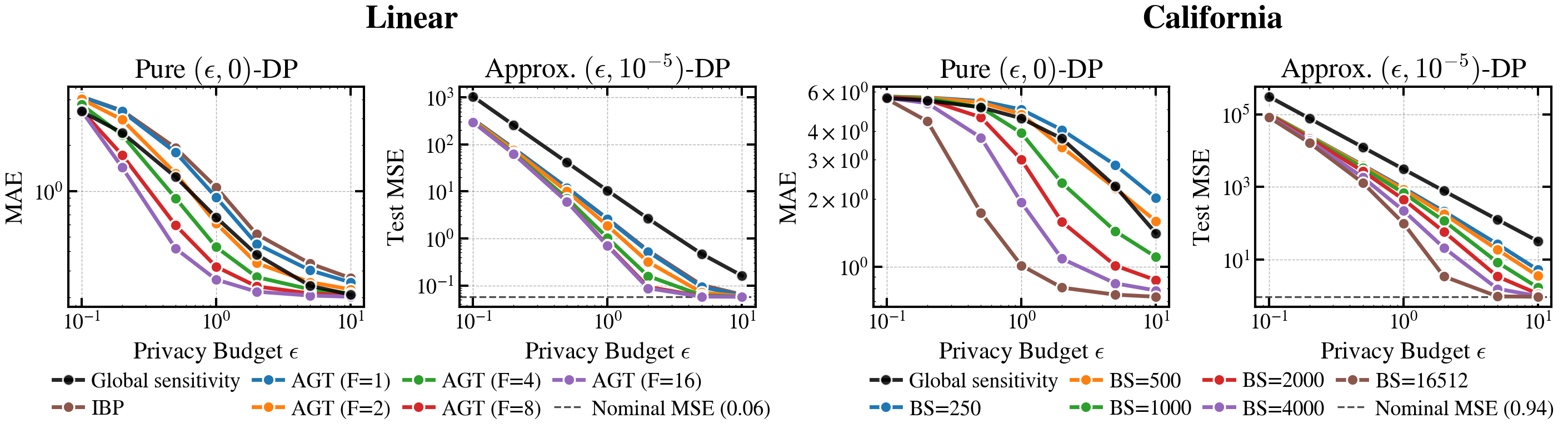}
    \caption{AGT-R ablations. Linear: the concretisation frequency $F$ of AGT against IBP; California: the batch size. Pure DP with the tighter Cauchy release at $\beta = \epsilon/2$ (MAE), approximate DP with Laplace at $\delta = 10^{-5}$ (test MSE), global sensitivity in black. Mean over 200 noise trials.}
    \label{fig:agtr_ablation}
  \end{subfigure}
  \caption{AGT-R on Linear and California under pure (left of each pair) and approximate (right) DP.}
  \label{fig:agtr_real}
\end{figure*}

\begin{theorem} \label{thm:1_step_ags_gauss}
    Consider the regressor $f$, dataset $D$, confidence level $1 - \alpha$, an absolute constant $c > 0$ from the subgaussian Hoeffding inequality, and a fixed failure probability $\delta \in (0,1)$ of approximate-DP. Assume the loss $\mathcal{L}$ is $\mu$-strongly convex and $L$-smooth with a locally constant Hessian across the $s_p$-step window, and that the step size satisfies $\eta \leq 1/L$. Then a single addition of Laplace noise with scale $b = 2\,\mathrm{SS}^\beta(f)/\epsilon$ and $\beta < \epsilon/(2\ln(2/\delta))$ (AGS) provides tighter error bounds than Gaussian DP-SGD with parameter $\sigma = s_p\Delta_2 f \sqrt{2\ln(1.25s_p/\delta)}/\epsilon$ (composing to $(\epsilon, \delta)$-DP across the window) in the sense that $E_L < E_G$, whenever:
    \begin{align*}
    \mathrm{SS}^\beta(f_{\theta},D) < \frac{1}{2\ln\left(\frac{1}{\alpha}\right)}   \sqrt{\frac{16\,\eta\,s_p^2(\Delta_2 f)^2 \ln\left(\frac{1.25 s_p}{\delta}\right)\ln\left(\frac{2}{\alpha}\right)}{3c\,\mu(2-\eta\mu)}}\!,
\end{align*}
    where $\mathrm{SS}^\beta(f,x)$ is the $\beta$-smooth sensitivity and $\delta$ is the allowed privacy leakage in approximate DP.
\end{theorem}
These theorem and subsequent proofs begin to demonstrate the core of the interplay between AGS and classical DP-SGD dynamics. Once we have access to the per-parameter error bounds, one can bound the total error for the whole set of parameters, albeit rather loosely, using the triangle inequality, or extend our approach using the Matrix Bernstein inequality in Thm. 5.4.1. of \cite{vershynin2025high}. Even having access to only the per-parameter utility error bounds, the performance degradation can be easily empirically measured by simply employing bound propagation techniques, such as IBP \cite{gowal2018effectiveness} or even AGT \cite{sosnin2025abstract} itself.

\subsection{Limitations and Future Work} 

AGS is one of the few algorithms to approach private learning from the lens of formal methods, and the first to employ parameter envelopes and frame learning as regression to satisfy differential privacy. Although performant in multiple scenarios, as we will demonstrate in \S\ref{sec:exp}, it comes with some limitations, which we discuss below.

\textbf{Computationally}, it is more expensive to obtain guarantees due to the fact that AGT costs roughly $4\times$ as much compared to a nominal training run, and the smooth sensitivity release requires one such pass per radius of the sensitivity staircase. Although in practice we require $< 10$ values for the staircase, this amounts to an order of magnitude increase in computational complexity. For tighter parameter envelopes, large batches significantly help, and so does more exact optimization (LP, MILP, MIQCP), which additionally increase the cost. \textbf{Scope-wise}, the noise of a single release grows with the number of released parameters, since the sensitivities of the individual coordinates add up in the certificate, so AGS is best suited to compact models, or to the trainable part of a larger pre-trained one, rather than to end-to-end training of deep networks. The parameter envelopes that the certificate is built on also loosen as they are propagated through non-linear layers, which makes deep models harder to certify tightly. \textbf{Regarding the privacy analysis}, the smooth sensitivity mechanism draws less benefit from the failure probability $\delta$ than mechanisms with data-independent noise such as the Gaussian mechanism, whose guarantees improve under composition; narrowing this gap would strengthen AGS under approximate differential privacy. Finally, each release along the training trajectory is currently accounted for as a separate mechanism and paid for from the same budget. A line of recent work shows that when only the final model is published, and the intermediate models are never observed, the privacy loss of an iterative algorithm can be far smaller than this composition suggests; bringing that view to the certified trajectory of AGS is the direction we find most promising.

\section{Experiments} \label{sec:exp}

In this section, we systematically  validate the strength of AGT-R and AGS in a variety of toy and real-world settings. 

\subsection{AGT-R} \label{ssec:agtr_exp}

\subsubsection{Condition Simulations}

We start by presenting a brief analysis of AGT-R from the lens of the conditions derived in Theorems \ref{thm:cauchy_lap_dom} and \ref{thm:laplace_gauss_dom}. We run simulations to assess at which ratios between the smooth and global sensitivity (i.e., $\rho = \text{SS}^\beta/\Delta_{p}f,\, p \in \{1,2\}$), for increasing confidence levels $1 - \alpha$, and (user-chosen) utility improvement constant $c$, the conditions hold. Figure \ref{fig:cond_table} shows the simulation: pure $(\epsilon, 0)$-DP in the left column and approximate $(\epsilon, \delta)$-DP in the right. 

The graph on the left hand side reveals that, in the best case, private predictions with Cauchy noise calibrated to the smooth sensitivity will yield smaller errors than Laplace-calibrated private predictions at confidence $1 - \alpha \approx 0.4$ and sensitivity ratio $\rho < 0.117$. This essentially means that, in order to preserve utility, our AGT-based upper bounds on $\text{SS}^\beta$ must be at least ten times smaller than the global sensitivity. However, this always holds for unbounded output spaces.

Turning to the approximate $(\epsilon, \delta)$-DP setting, we now restrict the plot to the high-confidence regime $1 - \alpha > 0.8$. The feasibility region is markedly more relaxed: since the reference is $\Delta_2 f \leq \Delta_1 f$ and both mechanisms are light-tailed, the admissible ratios exceed $1$, peaking at $\rho \approx 1.9$ for $c=1$. A ratio $\rho > 1$ means our bound need not even improve on the global sensitivity to win, so at $c=1$ the Laplace-calibrated smooth sensitivity is preferable across essentially the whole range, dropping below unity only past $1 - \alpha \approx 0.995$. A detailed analysis, including the constraints on $\beta$ that make these curves optimistic ceilings, is deferred to Appendix~\ref{app:feasibility}.

\begin{figure*}[t]
    \centering
    \includegraphics[width=\textwidth]{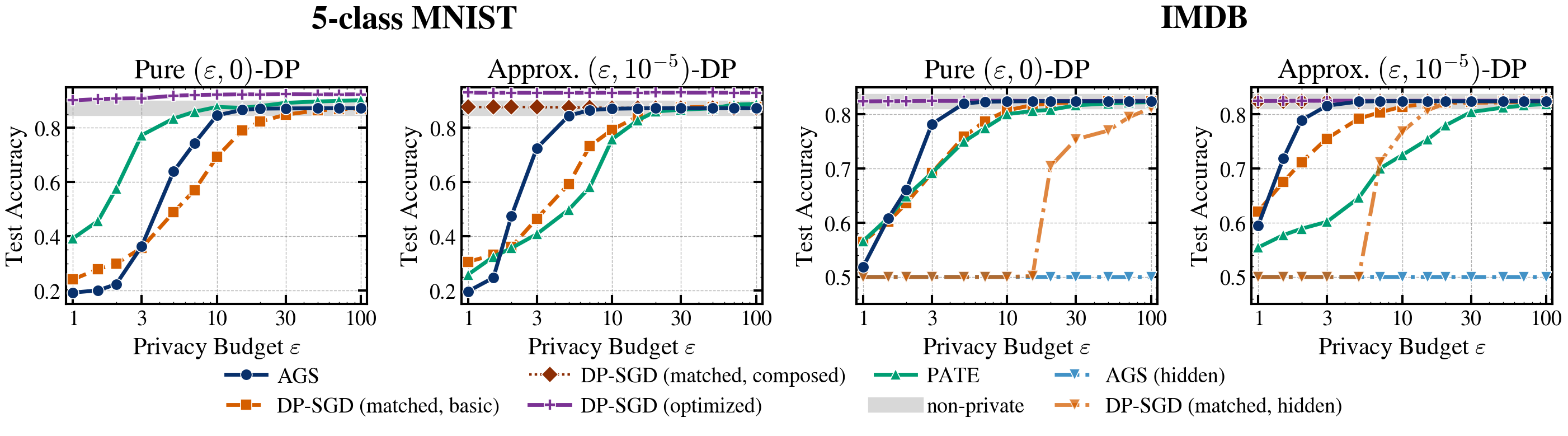}
    \caption{ Privacy--utility trade-offs of AGS for 5-class MNIST (left pair) and IMDB (right pair). AGS releases the trained parameters once (Cauchy mechanism under pure DP, Laplace under approximate DP, $\delta=10^{-5}$). Matched DP-SGD is accounted by basic composition or, under approximate DP, by exact composition of the Gaussian mechanism; optimized DP-SGD is tuned on a public selection set. On IMDB, the hidden-layer arms replace the logistic head with one hidden layer.}
    \label{fig:combined_mnist_imdb}
\end{figure*}

\subsubsection{Real-World Datasets}

For AGT-R, we consider two regression settings.
The first is a \emph{synthetic linear regression} task ($y = 2x + 1 + \xi$, 40{,}000 training points) learned by a linear model. The second is the California Housing dataset~\cite{pace1997cali, li2015acali, california_housing_sklearn}, comprising 20{,}640 records with 8 features that predict median house value, learned by an 8-64-1 ReLU network. We compare our novel method in the pure and approximate DP settings with data-independent private prediction, namely the Laplace and Gaussian mechanisms with global sensitivity. We additionally use as a baseline PATE \cite{papernot2016pate} with mean aggregation for regression (which we term PATE-R) and lastly equip each PATE shard with AGT-R capabilities and coin this technique PATE-AGT-R. For the two latter private prediction algorithms, we denote as $T$ the number of shards and report each at its best $T$ per budget. Lastly, we run AGT-R on a single shard with the budget amplified by sampling without replacement. Figure~\ref{fig:agtr_real} reports, for both datasets, the mean absolute error (MAE) under pure DP and the MSE under approximate DP ($\delta = 10^{-5}$): the top row compares AGT-R with the baselines above, the bottom row ablates AGT-R itself. Dataset, mechanism and training details are given in Appendix~\ref{app:hyperparam_agtr}.

\paragraph{Private prediction baselines (Figure~\ref{fig:agtr_pate}).}
Figure~\ref{fig:agtr_pate} reports the \emph{excess} error over the non-private prediction, which keeps arms that reach the non-private floor apart on the log axis; the raw scale with global sensitivity is Figure~\ref{fig:agtr_gs} (Appendix~\ref{app:agtr_gs}). There, across both datasets and settings, private prediction with global sensitivity ranks last in MAE (pure DP) and MSE (approximate DP): at $\epsilon = 1$, the best arm's error is $57$ times (linear) and $6$ times (California) lower in pure DP, and five and three orders of magnitude lower in approximate DP. Because of the task's simplicity, in the linear dataset case, AGT-R's and PATE's variants are essentially indistinguishable from a performance point of view, with one exception. Indeed, at a low budget ($\epsilon = 0.1$) in approximate DP, AGT-R on the full data is two orders of magnitude worse in terms of error; however, through dataset subsampling and amplification, this gap is quickly closed.

On the California dataset, the differences are more noticeable. While AGT-R on the full data is the least performant of the four arms due to full-data training, it still consistently outperforms data-independent private prediction (Appendix~\ref{app:agtr_gs}). Notably, at $\epsilon=1$ its error is $4.5$ times lower in pure DP and over $30$ times lower in approximate DP. Interestingly, the most accurate technique in pure DP is PATE-R, while the winner in approximate DP is PATE-AGT-R. The explanation for this is simple: because the Cauchy distribution is heavy-tailed and each point is the mean over 200 noise trials, the AGT-based arms tend to sample far-from-mean points more often, which means the utility degradation is more acute. In contrast to that, the addition of Laplace noise (which has finite second moments) and the $1/T$ reduction in sensitivity from averaging over shards essentially make PATE-AGT-R state-of-the-art, with test MSE within $0.06$ of the non-private one ($0.94$) at $\epsilon \geq 2$.

\paragraph{Batch and concretization frequency (Figure~\ref{fig:agtr_ablation}).}
For the linear regression ablation we train a linear model with SGD for 16 steps (batch size 5), and for California Housing a ReLU MLP (input--$64$--$1$) for 330 SGD steps at every batch size, both with gradient clipping at $\gamma = 0.1$; the global sensitivity is the clipping ceiling of the trajectory ($0.64$) for linear regression and the output range ($10$ standard deviations) for California Housing.

The reachability problem solved by AGT \cite{sosnin2025abstract} is solved over windows (termed concretization frequencies). A key advantage of our approach to privacy is that it exposes formal method techniques as explicit tuning knobs. We show that increasing the concretization frequency from interval bound propagation to a MILP every $F = 16$ steps leads to substantially tighter smooth sensitivity estimates, from about $0.25$ to $0.015$, giving a four- to six-fold improvement in utility.

For the California Housing dataset, we ablate the batch size $\texttt{BS}$, which is known to influence the AGT algorithm, up to the full batch ($\texttt{BS}=16{,}512$). As increasing the batch size tightens the bounds from AGT it should in turn tighten the smooth sensitivity upper bounds, as well as downstream utility. Indeed we see that smooth sensitivity and batch size appear roughly inversely proportional: \texttt{BS=250} peaks near $6$, \texttt{BS=4000} near $0.4$, while the full batch is nearly flat close to zero (Appendix~\ref{app:agtr_certificates}).

The critical point we would like to emphasize is that when compared with global sensitivity, our method is dominant in approximate DP at every privacy budget, frequency and batch size, with up to two orders of magnitude lower MSE. In pure DP, it dominates once the certificate is tight enough ($F \geq 8$ from $\epsilon = 0.2$ on linear regression, $\texttt{BS} \geq 1000$ from $\epsilon = 1$ on California Housing).

\subsection{AGS} \label{sec:exp_ags}

We evaluate AGS's performance with three distinct datasets: blobs \cite{scikit-learn} (two Gaussian clusters in 8 dimensions), MNIST \cite{lecun1998mnist} and the IMDB binary movie sentiment analysis task \cite{imdb}. We further augment this evaluation, in Appendix~\ref{app:sampling_abl}, with an investigation of the behaviour of AGS upon varying the number of releases (i.e., ``sampling'' steps) at a fixed privacy budget. Finally, we benchmark a hybrid DP-SGD+AGS approach on the SST2 dataset \cite{sst2} setup of \cite{dpft} and \cite{dpllm} against DP-LoRA finetuning.

\subsubsection{Utility} \label{ssec:ags_utility_exp}

For the results on the blobs dataset (200 points, linear hinge-loss classifier), which can be seen in Figure \ref{fig:blobs_ags}, the parameter envelopes are computed by solving an exact optimization problem, at different concretization frequencies $F$ (the number of training steps after which the MILP is solved). We compare AGS with PATE \cite{papernot2016pate} and matched (i.e., with the exact same parameters) DP-SGD. Looking at the approximate DP arm of the figure, there is strong evidence that AGS is the stronger mechanism: from $\epsilon = 5$ it leads both baselines, by more than $0.2$ in accuracy at $\epsilon = 20$, and approaches the non-private accuracy of $0.90$ at $\epsilon = 100$ whatever the concretization frequency. In pure DP, the frequency matters: solving the MILP every 10 steps instead of every step raises the accuracy at $\epsilon = 30$ from $0.59$ to $0.86$.

Figure \ref{fig:combined_mnist_imdb} shows how our approach fares in more difficult tasks from a performance perspective (i.e. task accuracy) compared to the same baselines as in blobs, with the addition of optimized DP-SGD, where we use amplification and composition to the greatest extent in trying to optimize for accuracy. We restrict MNIST to its first five digits (28{,}596 private training images), a multi-class problem on which the parameter envelope has to bound 45 parameters instead of the 9 of blobs, and on IMDB we train on 20{,}000 private reviews embedded by the frozen all-mpnet-base-v2 sentence encoder \cite{song2020mpnet, reimers2019sbert}, the family of pretrained text encoders that DP text pipelines build on \cite{xie2024augpe}, showing that AGS scales to foundation-model backbones by certifying only the model trained on top of them. In both cases this model is a linear head on 8 PCA components of the inputs (of the 768-dimensional sentence embeddings on IMDB). Without fail, this version wins across all our setups, showing performance on-par with the non-private benchmark even at $\epsilon=1$. Matched DP-SGD with accounting does so as well, in the approximate DP cases. As expected, due to the heavy-tailedness of the Cauchy distribution, AGS saturates later in the case of pure DP, reaching accuracies on-par with the non-private benchmark at $\epsilon = 15$ on 5-class MNIST and $\epsilon = 5$ on IMDB. With the exception of pure DP on MNIST, AGS consistently outperforms PATE from $\epsilon = 2$, notably showing $>0.1$ gains in accuracy for $1.5 \le \epsilon \le 7$ in approximate DP on IMDB. A similar behaviour is observed when comparing AGS with basic matched DP-SGD, whereby our method also wins in pure DP on MNIST from $\epsilon = 5$.

\begin{figure}[!tp]
 \centering
 \includegraphics[width=0.48\textwidth]{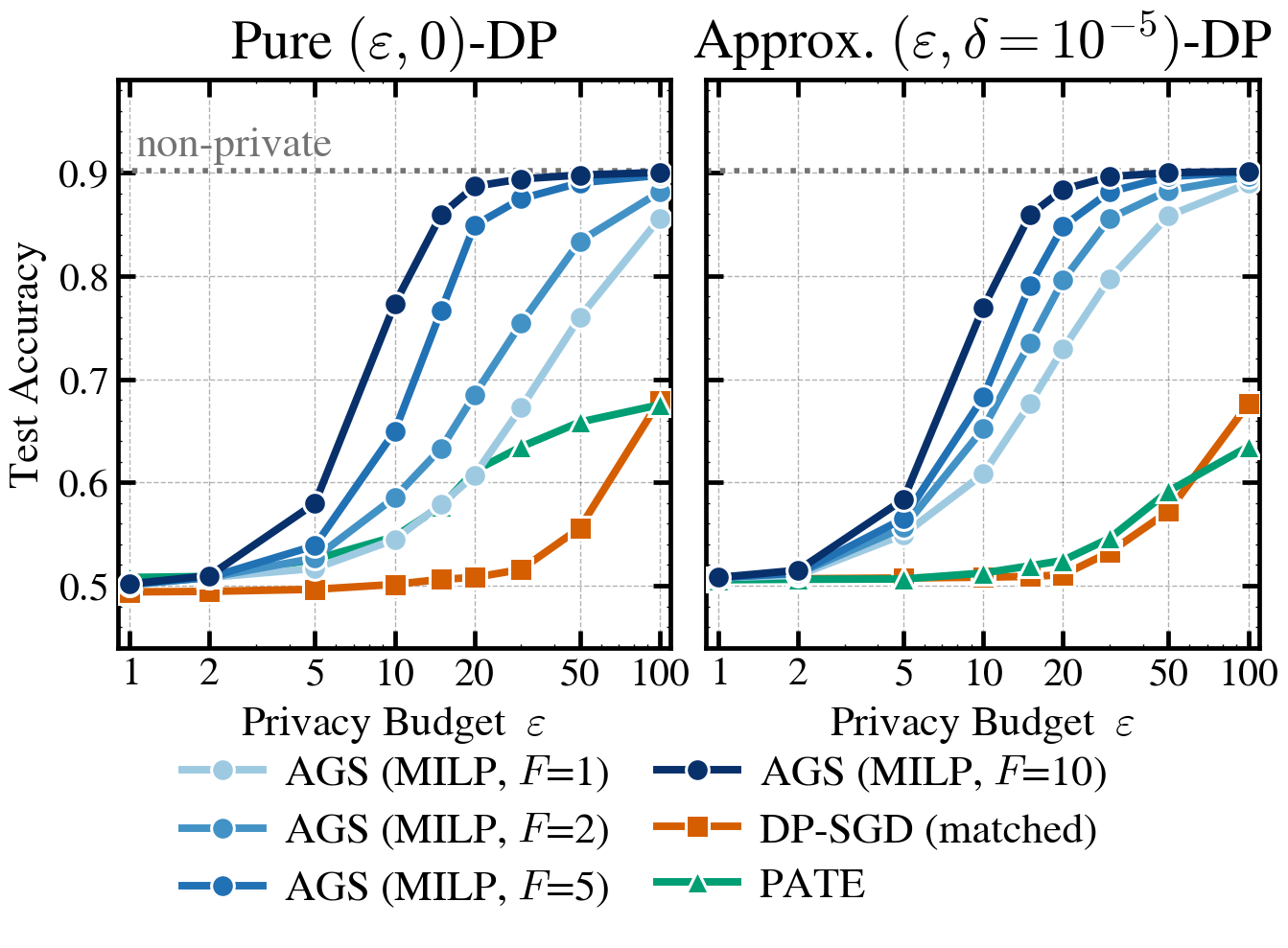}
 \caption{Privacy--utility trade-offs of AGS on blobs. AGS certifies the parameter envelope with an exact MILP solved every $F$ training steps and releases the trained parameters once (Cauchy mechanism under pure DP, Laplace under approximate DP, $\delta=10^{-5}$). DP-SGD (matched) uses the same hyperparameters with basic composition. The dotted line marks the non-private accuracy. Medians over 100 noise draws.}
 \label{fig:blobs_ags}
\end{figure}
\begin{figure}[!tp]
 \centering
 \includegraphics[width=0.48\textwidth]{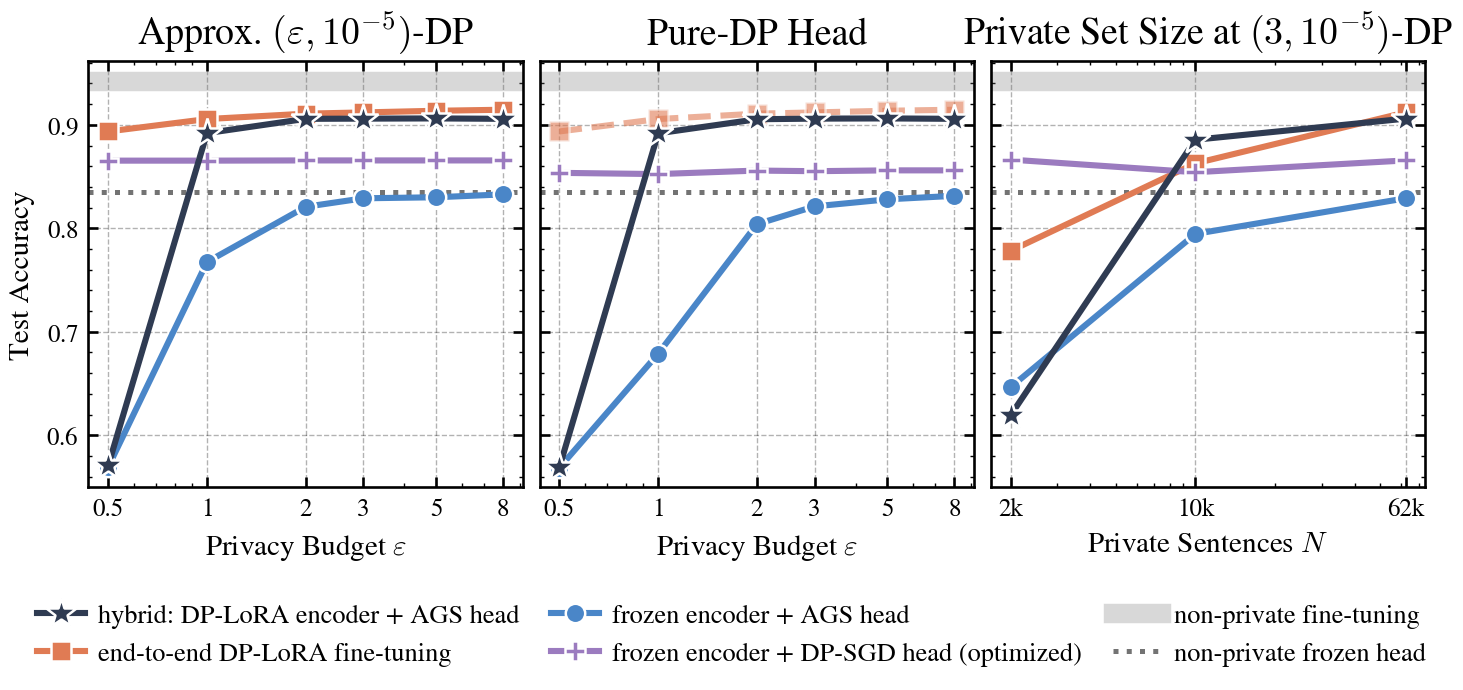}
 \caption{Hybrid approach on SST-2: a DP-LoRA fine-tuned encoder with an AGS head against end-to-end DP-LoRA fine-tuning and a frozen encoder with an AGS or a DP-SGD head. Test accuracy against the total budget $\epsilon = \epsilon_1 + \epsilon_2$ under $(\epsilon, 10^{-5})$-DP (left), a pure-DP head (middle), and the private set size $N$ at $\epsilon = 3$ effect on accuracy (right). }
 \label{fig:hybrid_ags_dpsgd}
\end{figure}

\subsection{Hybrid Approach}

We employ the SST-2 dataset \cite{sst2} to show that AGS can work well in conjunction with DP-SGD and reaches performance within 4 points of non-private fine-tuning ($94.3\%$), which is also on-par with the best full DP-SGD method. As in \S\ref{ssec:ags_utility_exp}, we use the all-mpnet-base-v2 sentence encoder on the binary sentiment analysis task induced by our dataset, whose 67,349 training sentences we split into 5,000 held-out test sentences and 62,349 private training instances, with the 872-sentence validation split as the public set.

The encoder is fine-tuned with DP-LoRA \cite{dpft} (rank 16) using DP-SGD in Opacus \cite{yousefpour2021opacus} for 3 epochs at batch size 1,024 to $(\epsilon_1, 5\cdot 10^{-6})$-DP, and the same model scored with its own head is the end-to-end DP-LoRA arm. The AGS head is a logistic head on the PCA-8 projection of the adapted features ($p = 18$), released once at $(\epsilon_2, 5\cdot 10^{-6})$-DP, so that the pipeline is $(\epsilon_1 + \epsilon_2, 10^{-5})$-DP; the split and the decision threshold are chosen on the public set (full details in Appendix \ref{app:hyperparam_ags}). While most settings tested we find that DP-LoRA is the best approach, in Figure~\ref{fig:hybrid_ags_dpsgd} we do observe that hybrid DP-LoRA+AGS can outperform both DP-LoRA and AGS alone when the number of sentences is at 10k.

\section{Conclusion}

In this work, we proposed two novel differentially private algorithms --- one for private prediction and one for private learning --- grounded in formal verification via Abstract Gradient Training. Firstly, we formally introduced AGT-R, a novel private prediction method that guarantees differential privacy in arbitrarily complex regression settings. We characterized AGT-R both theoretically and empirically and provided closed-form conditions for when our formal verification-based upper-bound on smooth sensitivity yields better utility-privacy trade-offs. Notably, we tackle private regression without making additional assumptions on the output space, namely unboundedness or discreteness. We further leveraged our insights gathered from AGT-R to pose learning as a private regression problem, and use the same bounds originating from AGT to devise a novel private learning algorithm, AGS. We demonstrate theoretically how AGS satisfies (pure or approximate) differential privacy by sampling at predefined user-chosen steps from an appropriate noise distribution. Once again, we provide a closed form expression for the case when AGS provides better utility than standard DP-SGD. Lastly, we verify all our claims through carefully-curated experiments that show both the inner mechanisms workings of our proposed approaches, as well as how they compare to prior established DP-ensuring techniques.

\bibliographystyle{IEEEtran}
\bibliography{IEEEabrv,reflist}
\clearpage

\appendices
\raggedbottom

\section{Theoretical Derivation of Parameter-Space Bounds} \label{app:agt_bounds}

To construct the parameter envelope $T_k = [\theta^k_L, \theta^k_U]$ required by AGT-R, we must bound the gradient updates at each optimization step. Consider a regression model trained with a loss function $\mathcal{L}(\cdot, \cdot)$ and a gradient clipping threshold $\gamma$. Let $\mathcal{M}$ represent our gradient-based training algorithm and $\theta_{\text{init}}$ be a fixed initialization. The valid parameter-space bound $T_k$ must satisfy:

\[ \mathcal{M}(f, \theta_{\text{init}}, D') \in T_k \quad \forall D' \text{ s.t. } d(D, D') \leq k. \]

At each training iteration on a nominal batch $B$ of size $b$, the set of possible descent directions under $k$ arbitrary dataset additions or removals is bounded using either bound propagation (IBP \cite{gowal2018effectiveness}) or encoding the problem as an optimization (LP/QP/MILP/MIQCP). We bound the perturbed update $\Delta \theta \in [\theta^k_L, \theta^k_U]$ element-wise as follows:

\begin{align}
    \Delta \theta_L &= \frac{1}{b} \left( \text{SEMin}_{b-k} \{ \delta_L^{(i)} \} - k \gamma \mathbf{1}_d \right), \\
    \Delta \theta_U &= \frac{1}{b} \left( \text{SEMax}_{b-k} \{ \delta_U^{(i)} \} + k \gamma \mathbf{1}_d \right),
\end{align}
where $\text{SEMin}$ and $\text{SEMax}$ represent the sum of the element-wise bottom and top $b-k$ gradient bounds. The terms $\delta_L^{(i)}$ and $\delta_U^{(i)}$ are bound propagation- or optimization-derived bounds on the gradient of the loss $\mathcal{L}$ for the $i$-th sample with respect to the reachable parameters $\tilde{\theta} \in T_{k-1}$. The $k\gamma \mathbf{1}_d$ term accounts for the worst-case scenario where the $k$ differing points produce gradients exactly at the clipping threshold $\gamma$ in the adversarial direction.

By iteratively applying these worst-case updates across the entire optimization trajectory, we obtain the final envelope $T_k$. To compute the upper bound on the smooth sensitivity for a regression output, we propagate the test query $x$ through the network over $T_k$. The maximum output variation provides our certified bound:
\[
A_k(f,x) = \max_{\tilde{\theta} \in T_k} \left\| f_{\tilde{\theta}}(x) - f_{\theta}(x) \right\|_1 \;\ge\; A_k(f,x).
\]

By substituting this upper bound into the AGT-R noise mechanism, we maintain $(\epsilon,\delta)$-differential privacy while accounting for worst-case local sensitivity.

\section{Proofs} \label{app:proofs}

\subsection{Proof of Theorem \ref{thm:cauchy_lap_dom}} \label{app:proof_agtr_cauchy_dom}

\begin{proof}
Let $Z_L \sim \text{Lap}(b)$ and $Z_C \sim \text{Cauchy}(\gamma)$ with $\gamma, b$ defined as above. The worst case error $E_L$ at an $\alpha$ quantile for the Laplace noise is $\mathbb{P}(|Z_L| > E_L) \leq \alpha$. Expanding using the CDF and remembering that the Laplace noise is symmetric about $\mu = 0$, then $\mathbb{P}(|Z_L| > E_L) = 2 \cdot \mathbb{P}(Z_L > E_L) = 2[1 - (1- (1/2)\exp(-E_L/b))] = \exp(-E_L/b) \leq \alpha$. Rearranging, this yields a worst-case $E_L = \frac{\Delta f}{\epsilon}\ln(\frac{1}{\alpha})$. We proceed analogously in the case of Cauchy noise:  
$\mathbb{P}(|Z_C| > E_C) = 2 \cdot \mathbb{P}(Z_C > E_C) = 1 - (2 / \pi) \arctan(E_C /\gamma) \leq \alpha$. We thus find an upper bound $E_C = \frac{6\, \text{SS}^\beta}{\epsilon} \tan\left(\frac{\pi}{2}(1 - \alpha)\right)$. Enforcing that $E_C < \frac{1}{c}E_L$ gives the desired result.
\end{proof}

\subsection{Proof of Theorem \ref{thm:laplace_gauss_dom}} \label{app:proof_agtr_lap_dom}

\begin{proof}
 Let $Z_G \sim \mathcal{N}(0, \sigma^2)$ with $\sigma$ defined as above, and $Z_L \sim \mathrm{Lap}(2\, \mathrm{SS}^\beta / \epsilon)$. The worst case error $E_G$ at $\alpha$ is $\mathbb{P}(|Z_G| > E_G) \leq \alpha$. Since the Gaussian is symmetric around the mean and by standardizing $Z_G$ (i.e., $Z_G / \sigma \sim \mathcal{N}(0,1)$), we have that $2\cdot\mathbb{P}\left((Z_G/\sigma) > (E_G/\sigma)\right) \leq \alpha$ and thus $1 - \Phi(E_G/\sigma) \leq \alpha/2$. Rearranging, we obtain a worst case $E_G = \sigma \Phi^{-1}(1 - \frac{\alpha}{2})$. For Laplace noise, we proceed analogously to Thm. \eqref{thm:cauchy_lap_dom}: $\mathbb{P}(|Z_L| > E_L) \leq \alpha$ and setting our condition to be $E_L < E_G$. By symmetry, we have that $\mathbb{P}(|Z_L| > E_L) = 2 \cdot \mathbb{P}(Z_L > E_L) = 2[1 - (1- (1/2)\exp(-E_L/b))] = \exp(-E_L/b) \leq \alpha$. Once again, rearranging, we obtain a worst case $E_L = b\ln (\frac{1}{\alpha}) = \frac{2 \, \mathrm{SS}^\beta}{\epsilon}\ln (\frac{1}{\alpha})$. We enforce $E_L < \frac{1}{c}E_G$ to obtain the final result. 
\end{proof}

\subsection{Proof of Corollary~\ref{thm:1_step_ags}}\label{app:admisibillity_proof}

Recall \cite[Def.~2.4]{nissim2007smooth}: a density $h$ on $\mathbb{R}^p$ is $(\alpha,\beta)$-admissible if, for all $\Delta$ with $\|\Delta\|_1 \leq \alpha$ (the norm in which \cite{nissim2007smooth} define local sensitivity), all $|\lambda| \leq \beta$ and all measurable $S$,
$\Pr[Z \in S] \leq e^{\epsilon/2}\Pr[Z \in S + \Delta] + \delta/2$ (sliding) and
$\Pr[Z \in S] \leq e^{\epsilon/2}\Pr[Z \in e^{\lambda} S] + \delta/2$ (dilation), where $Z \sim h$.

By \cite[Lemma~2.5]{nissim2007smooth}, $\theta^{n_s} + (u/\alpha)\,Z$ is then $(\epsilon,\delta)$-indistinguishable whenever $u$ is a $\beta$-smooth upper bound on the $\ell_1$ local sensitivity of $\theta^{n_s}$, which $u(\beta)$ is by construction. It therefore suffices to show that $h(z) = \prod_{j=1}^p h_1(z_j)$ is $(\epsilon/6,\, \epsilon/(6p))$-admissible with $\delta = 0$ for $h_1(z) = 1/(\pi(1+z^2))$, and $(\epsilon/2,\, \epsilon/(2p(1+2\ln(2p/\delta))))$-admissible for $h_1(z) = \tfrac12 e^{-|z|}$; the scales $u/\alpha$ are then $6u/\epsilon$ and $2u/\epsilon$. Both properties follow from a bound on the log-density ratio, which is a sum over coordinates because $h$ is a product: a ratio at most $e^{\epsilon/2}$ on a set $B$ with $\Pr[Z \notin B] \leq \delta/2$ integrates to the required inequality.

\emph{Sliding.} $|\tfrac{d}{dz}\ln h_1| \leq 1$ for both densities ($\tfrac{2|z|}{1+z^2} \leq 1$, resp.\ $1$), so $\ln h(z) - \ln h(z + \Delta) \leq \sum_j |\Delta_j| = \|\Delta\|_1 \leq \alpha \leq \epsilon/2$ for every $z$.

\emph{Dilation.} The density of $e^{-\lambda}Z$ is $e^{p\lambda}h(e^{\lambda}z)$, so the log ratio is $\sum_{j=1}^p \big[\ln h_1(z_j) - \ln h_1(e^{\lambda} z_j) - \lambda\big]$.
(C) The $j$-th term is $g_j(\lambda) - g_j(0) - \lambda$ with $g_j(\lambda) = \ln(1 + z_j^2 e^{2\lambda})$ and $g_j' \in [0,2)$, hence at most $3|\lambda| \leq 3\beta$ in absolute value; the sum is at most $3p\beta = \epsilon/2$ for every $z$, so $\delta = 0$.
(L) The $j$-th term is $|z_j|(e^{\lambda} - 1) - \lambda \leq \beta(1 + 2|z_j|)$, using $e^{\beta} - 1 \leq 2\beta$ for $\beta \leq 1$ (which holds for every $\epsilon \leq 2p$). Under $h$ the $|z_j|$ are i.i.d.\ $\mathrm{Exp}(1)$, so $B = \{\max_j |z_j| \leq \ln(2p/\delta)\}$ has $\Pr[Z \notin B] \leq p \cdot \delta/(2p) = \delta/2$ by a union bound, and on $B$ the sum is at most $p\beta\,(1 + 2\ln(2p/\delta)) = \epsilon/2$. \qed

\subsection{Proof of Theorem \ref{thm:1_step_ags_cauchy}} \label{app:proof_ags_pure}
 We start by computing a bound on the error term in the case of Laplace DP-SGD (which we will refer to, alternatively, as noisy SGD). The clean and noisy updates are: 
\begin{align*}
    \theta_{t} &= \theta_{t-1} - \eta \, \partial_{\theta_{t-1}} \mathcal{L}(\bm{\theta}, D) \\
    \tilde{\theta}_{t} &= \tilde{\theta}_{t-1} - \eta \left(\partial_{\tilde{\theta}_{t-1}} \mathcal{L}(\bm{\tilde{\theta}}, D) + Z_t\right), \quad Z_t \sim \mathrm{Lap}\!\left(\frac{s_p \Delta f}{\epsilon}\right),
\end{align*}
yielding the error term $e_t = \tilde{\theta}_t - \theta_t$: 
\[
    e_t = (\tilde{\theta}_{t-1} - \theta_{t-1}) - \eta \left( \partial_{\tilde{\theta}_{t-1}} \mathcal{L}(\bm{\tilde{\theta}}, D) - \partial_{\theta_{t-1}} \mathcal{L}(\bm{\theta}, D) \right) - \eta Z_t.
\]
Noting that $e_{t-1} = \tilde{\theta}_{t-1} - \theta_{t-1}$, we expand the gradient of the loss function with respect to $\theta_{t-1}$ at $\tilde{\theta}_{t-1}$ using Taylor's theorem with the Lagrange remainder to obtain: 
\begin{align*}
    \partial_{\theta_{t-1}} \mathcal{L}(\bm{\theta}, D) 
    &= \partial_{\tilde{\theta}_{t-1}} \mathcal{L}(\bm{\tilde{\theta}}, D) \\
    &\quad + \partial^2_{\tilde{\theta}_{t-1}} \mathcal{L}(\bm{\tilde{\theta}}, D) (\theta_{t-1} - \tilde{\theta}_{t-1}) \\
    &\quad + \mathcal{O}\left( \|\theta_{t-1} - \tilde{\theta}_{t-1}\|^2 \right).
\end{align*}
Rearranging, dropping the higher-order terms and noting that $\tilde{\theta}_{t-1} - \theta_{t-1} = e_{t-1}$, we obtain 
\begin{align*}
    \partial_{\tilde{\theta}_{t-1}} \mathcal{L}(\bm{\tilde{\theta}}, D) - \partial_{\theta_{t-1}} \mathcal{L}(\bm{\theta}, D) = \partial^2_{\tilde{\theta}_{t-1}} \mathcal{L}(\bm{\tilde{\theta}}, D) e_{t-1}.
\end{align*}
Substituting this back, the recursive error term is 
\begin{align*}
    e_t & \approx e_{t-1} - \eta \partial^2_{\tilde{\theta}_{t-1}} \mathcal{L}(\bm{\tilde{\theta}}, D) e_{t-1} - \eta Z_t \\ 
    & = \left(1 - \eta \partial^2_{\tilde{\theta}_{t-1}} \mathcal{L}(\bm{\tilde{\theta}}, D)\right) e_{t-1} - \eta Z_t.
\end{align*}
Iteratively expanding the provided expression, we obtain the $s_p$-window error term 
\[
e_{s_p} = -\eta \left(\sum_{s=1}^{s_p} Z_s \cdot\prod_{r = s+ 1}^{s_p} \left(1 - \eta \, h_r\right)\right),
\]
whereby we denoted by $h_r$ (in the vein of a `'Hessian'`) the second derivative of the loss function $\partial^2_{\tilde{\theta}_{r-1}} \mathcal{L}(\bm{\tilde{\theta}}, D)$ with respect to the noisy parameter. We take $e_0 = 0$, so that the first injected noise is $Z_1$. Taking the (Euclidean) $1$-norm of the error term (i.e. absolute value, since we are in $1$ dimension) we have that 
\[
|e_{s_p}| = \left|\sum_{s=1}^{s_p} Z_s \cdot (- \eta) \prod_{r = s+ 1}^{s_p} \left(1 - \eta \, h_r\right)\right|.
\]
Let us, for now, denote $a_s = (-\eta)\prod_{r = s+ 1}^{s_p} \left(1 - \eta \, h_r\right)$, $a = (a_1,\dots,a_{s_p})$ and analyze $\left|\sum_{s=1} ^{s_p} a_s Z_s\right|$ in isolation. In particular, notice that the Laplace distribution is a \textit{subexponential} distribution and note the subexponential version of the Bernstein inequality in  Theorem 2.9.1. of \cite{vershynin2025high}. We assume the curvature is (locally) constant across the window, i.e.\ $h_r \equiv h$ for $r \in \{1,\dots,s_p\}$, consistent with the first-order Taylor approximation employed above. Consequently the coefficients $a_s$ are deterministic and independent of the noise variables $\{Z_s\}$, which is precisely the condition required to apply the weighted subexponential Bernstein inequality (Corollary~2.9.2 of \cite{vershynin2025high}) to $\sum_s a_s Z_s$. Using Corollary 2.9.2. of the same source \cite{vershynin2025high}, namely the weighted version of the subexponential Bernstein inequality, we obtain: 

\begin{align*}
    \mathbb{P} \left\{|e_{s_p}| \geq E_L \right\} &= \mathbb{P}\left\{\left|\sum_{s=1} ^{s_p} a_s Z_s\right| \geq E_L\right\} \\
    &\leq 2 \exp\left[-c \min \left(
    \frac{E_L^2}{K^2\|a\|_2^2}, \frac{E_L}{K\|a\|_{\infty}}
    \right)\right],
\end{align*}
where $c>0$ is an absolute constant and $K = \max_s\|Z_s\|_{\psi_1}$, with $\|\cdot\|_{\psi_1}$ the subexponential norm (i.e. $\|Z\|_{\psi_1} = \inf \{K >0: \mathbb{E}[\exp(|Z| / K)] \leq 2\}$). Since $Z_s$ are all i.i.d. Laplace with mean $0$ and scale parameter $b = (s_p \Delta f) / \epsilon$, then $\forall s: K = \|Z_s\|_{\psi_1}$, which can easily be computed as follows:

\begin{align*}
\mathbb{E}[\exp(|Z_s| / K)] &= \int_{\mathbb{R}} \exp\left(\frac{|z|}{k}\right) \frac{1}{2b}\exp\left(\frac{-|z|}{b}\right) dz \\
&= \frac{1}{b} \int_0^\infty \exp \left(-z\frac{k - b}{bk}\right)dz  \\
&= \frac{k}{b - k}\exp \left(-z\frac{k - b}{bk}\right)\Bigg |_0^\infty \\ 
&= \frac{k}{k - b},
\end{align*}
where we have enforced in the second equality that $k > b$ and flipped the sign of $z$, otherwise the integral would not converge. The expression is bounded by $2$ when $\|Z_s\|_{\psi_1} = 2b = 2(s_p\Delta f) /\epsilon$.

We now compute $\|a\|_2^2$ and $\|a\|_\infty$. Firstly, we remind the reader that we have assumed $\mu$-strong convexity and $L$-smoothness, thus $\mu \leq h_r \leq L$, and we take the step size $\eta \leq 1/L$ so that $0 \leq 1 - \eta h_r \leq 1 - \eta\mu$. Therefore, $|a_s| \leq \eta(1 - \eta \mu)^{s_p - s}$ and $\|a\|_\infty = \max_s|a_s| \leq  \eta(1-\eta\mu)^0 = \eta$ (i.e. maximization happens when $s = s_p$).

We bound the $2$-norm in a similar way: $\|a\|_2^2 = \sum_s a_s^2 \leq \eta^2 \sum_{s=1}^{s_p} (1 -\eta\mu)^{2(s_p - s)}$. The latter is a geometric progression with ratio $r = (1 - \eta \mu)^2$, thus $\|a\|_2^2 \leq \eta^2\frac{1 - (1 -\eta \mu)^{2s_p}}{1 - (1 -\eta\mu)^2 } \leq \frac{\eta^2}{1 - (1-\eta\mu)^2} = \frac{\eta}{\mu(2 -\eta\mu)}$. The last inequality follows since $(1-\eta\mu)^{2s_p} \geq 0$, so dropping it only increases the numerator, and the final equality uses $1 - (1-\eta\mu)^2 = \eta\mu(2 - \eta\mu)$.

Substituting $K=2(s_p\Delta f) / \epsilon$ alongside our computed norms into the subexponential Bernstein inequality
\begin{align*}
 \mathbb{P}\!\left\{|e_{s_p}| \geq E_L \right\} & \leq 2\exp\!\Big[-\!c
 \;\cdot \\
 & \cdot \min\!\left(
 \frac{\epsilon^2 E_L^2\mu(2-\eta\mu)}{4\eta (s_p\Delta f)^2},\;
 \frac{\epsilon E_L}{2\eta s_p\Delta f}\right)\Big].
\end{align*}
To follow the same analysis as in Thm. \ref{thm:cauchy_lap_dom}, we express the error condition $E_L$ in terms of a user-defined confidence $\alpha$. Let us name the first argument in the minimum as $A$ and the second as $B$; then the worst case error at an $\alpha$ quantile is: $2 \exp(-c \min(A,B)) \leq \alpha$. This results in $\min(A, B) \geq \frac{-\ln(\alpha/2)}{c}$, so $A \geq \frac{-\ln(\alpha/2)}{c}$ and $B \geq \frac{-\ln(\alpha/2)}{c}$ simultaneously. We now solve for the \textit{smallest} $E_L$ satisfying each argument of the
minimizer, naming $E_{L,1}$ the solution w.r.t.\ $A$ and $E_{L,2}$ the solution
w.r.t.\ $B$, and obtain:
\[
E_{L,1} = \sqrt{\frac{4(s_p\Delta f)^2\eta}{\epsilon^2c\mu(\eta\mu - 2)}\ln{\frac{\alpha}{2}}}\;\; \text{and} \;\;
E_{L,2} = \frac{-2\eta s_p\Delta f \ln{\frac{\alpha}{2}}}{\epsilon c}.
\]
Since both constraints must hold simultaneously, the worst-case error is the
larger of the two, i.e.\ $E_L = \max(E_{L,1}, E_{L,2})$.
The error term in the Cauchy case is exactly as in the proof of Thm \ref{thm:cauchy_lap_dom}:
\[
E_C = \frac{6\,\mathrm{SS}^\beta}{\epsilon}\tan\left(\frac{\pi}{2}(1-\alpha)\right).
\]
Enforcing $E_C < E_L$ yields the required conditions under which AGS dominates
Laplace DP-SGD.
\qed

\subsection{Proof of Theorem \ref{thm:1_step_ags_gauss}} \label{app:proof_ags_approx}

We proceed exactly as in the proof of Thm.~\ref{thm:1_step_ags_cauchy}, the only difference being that we now perturb the clean update with Gaussian noise and compare against a one-step addition of Laplace noise calibrated to smooth sensitivity (i.e., we now find ourselves in the Approximate DP case). The clean and noisy updates are:
\begin{align*}
    \theta_{t} &= \theta_{t-1} - \eta \, \partial_{\theta_{t-1}} \mathcal{L}(\bm{\theta}, D) \\
    \tilde{\theta}_{t} &= \tilde{\theta}_{t-1} - \eta \left(\partial_{\tilde{\theta}_{t-1}} \mathcal{L}(\bm{\tilde{\theta}}, D) + Z_t\right), \quad Z_t \sim \mathcal{N}(0, \sigma^2),
\end{align*}
where $\sigma = \frac{s_p \Delta_2 f}{\epsilon}\sqrt{2\ln(1.25 s_p/\delta)}$ is the noise scale of the Gaussian mechanism run at the per-step budget $(\epsilon/s_p,\, \delta/s_p)$, so that by basic composition the $s_p$-step window costs $(\epsilon,\delta)$ in total, matching the single AGS release. Since the Taylor expansion of the gradient and the subsequent unrolling of the recursion are \textit{identical} to the Laplace case, we omit them here and pick up directly at the $s_p$-window error term:
\[
    e_{s_p} = -\eta \left(\sum_{s=1}^{s_p} Z_s \cdot \prod_{r = s+1}^{s_p} \left(1 - \eta \, h_r\right)\right),
\]
where, as before, $h_r$ denotes the `'Hessian'` $\partial^2_{\tilde{\theta}_{r-1}} \mathcal{L}(\bm{\tilde{\theta}}, D)$. Reusing the notation $a_s = (-\eta)\prod_{r = s+1}^{s_p} \left(1 - \eta \, h_r\right)$ and $a = (a_1, \dots, a_{s_p})$, we once again analyze $\left|\sum_{s=1}^{s_p} a_s Z_s\right|$ in isolation. As in the Laplace case, we assume the curvature is locally constant across the window (i.e. $h_r \equiv h$), so that the weights $a_s$ are deterministic and independent of the noise $\{Z_s\}$, as required by the concentration inequality we are about to invoke.

In contrast to the Laplace case, however, the Gaussian distribution is \textit{subgaussian}, so we appeal to the general Hoeffding inequality for subgaussian random variables in Theorem 2.7.3 of \cite{vershynin2025high}. Before doing so, we remind the reader that a constant $a$ multiplied with a Gaussian is again Gaussian, namely $a \cdot \mathcal{N}(m, \sigma^2) \sim \mathcal{N}(am, a^2\sigma^2)$ (this follows immediately by inspecting the moment generating function). Hence each weighted term is itself a centered Gaussian,
\[
    Z'_s := a_s Z_s \sim \mathcal{N}\!\left(0, \, \eta^2 \prod_{r=s+1}^{s_p}(1 - \eta h_r)^2 \, \sigma^2\right),
\]
and the $Z'_s$ are independent, centered and subgaussian, exactly the regime in which the inequality applies. We obtain:
\begin{align*}
    \mathbb{P} \left\{|e_{s_p}| \geq E_G \right\} &= \mathbb{P}\left\{\left|\sum_{s=1}^{s_p} a_s Z_s\right| \geq E_G\right\} \\
    &\leq 2 \exp\left(-\frac{c \, E_G^2}{\sum_{s=1}^{s_p} \|Z'_s\|_{\psi_2}^2}\right),
\end{align*}
where $c > 0$ is an absolute constant and $\|\cdot\|_{\psi_2}$ denotes the subgaussian norm (i.e. $\|X\|_{\psi_2} = \inf\{t > 0 : \mathbb{E}[\exp(X^2/t^2)] \leq 2\}$). The subgaussian norm $\|\cdot\|_{\psi_2}$ of a centered Gaussian can be computed exactly as in the previous proof (mirroring the $\psi_1$-norm derivation for Laplace); by \cite{vershynin2025high}, for $X \sim \mathcal{N}(0, \sigma^2)$ it equals $\|X\|_{\psi_2} = \sigma\sqrt{8/3}$. By the absolute homogeneity of the norm, it follows that $\|Z'_s\|_{\psi_2} = |a_s| \, \sigma\sqrt{8/3}$, and therefore $\sum_{s=1}^{s_p} \|Z'_s\|_{\psi_2}^2 = \tfrac{8}{3}\sigma^2 \|a\|_2^2$. 

We now note that $\|a\|_2^2$ is precisely the quantity we bounded in the proof of Thm.~\ref{thm:1_step_ags_cauchy}, namely $\|a\|_2^2 \leq \frac{\eta}{\mu(2 - \eta\mu)}$ (which holds under $\mu$-strong convexity, $L$-smoothness and the step-size condition $\eta \leq 1/L$). Substituting our computed norm alongside 
$\sigma^2 = \left[2s_p^2 (\Delta_2 f)^2 \ln(1.25 s_p /\delta)\right] / \epsilon^2$ into the inequality yields:
\[
    \mathbb{P} \left\{|e_{s_p}| \geq E_G \right\} \leq 2\exp\left(-\frac{3 c \, E_G^2 \, \mu(2 - \eta\mu) \, \epsilon^2}{16 \, \eta \, s_p^2 (\Delta_2 f)^2 \ln(1.25 s_p/\delta)}\right).
\]
To follow the same analysis as before, we express the error $E_G$ in terms of a user-defined confidence $\alpha$. Setting the probability bound to $\alpha$ and solving for $E_G$ gives the (single-regime) worst-case error
\[
    E_G = \sqrt{\frac{16 \, \eta \, s_p^2 (\Delta_2 f)^2 \ln\!\left(\frac{1.25 s_p}{\delta}\right) \ln\!\left(\frac{2}{\alpha}\right)}{3 c \, \mu(2 - \eta\mu) \, \epsilon^2}}.
\]

The error of AGS in this setting stems from a \textit{single} addition of Laplace noise calibrated to the $\beta$-smooth sensitivity (with $\beta < \epsilon/(2\ln(2/\delta))$ ensuring approximate $(\epsilon, \delta)$-DP), exactly as in the proof of Thm.~\ref{thm:laplace_gauss_dom}:
\[
    E_L = \frac{2\,\mathrm{SS}^\beta}{\epsilon}\ln\left(\frac{1}{\alpha}\right).
\]
Enforcing $E_L < E_G$ and solving for the smooth sensitivity yields the required conditions under which Laplace AGS dominates Gaussian DP-SGD:
\begin{align*}
    \mathrm{SS}^\beta(f_{\theta},D) < \frac{1}{2\ln\left(\frac{1}{\alpha}\right)}   \sqrt{\frac{16\,\eta\,s_p^2(\Delta_2 f)^2 \ln\left(\frac{1.25 s_p}{\delta}\right)\ln\left(\frac{2}{\alpha}\right)}{3c\,\mu(2-\eta\mu)}}\!,
\end{align*}
\qed

\section{AGT-R Ablations}

\begin{figure*}[t]
    \centering
    \includegraphics[width=\textwidth]{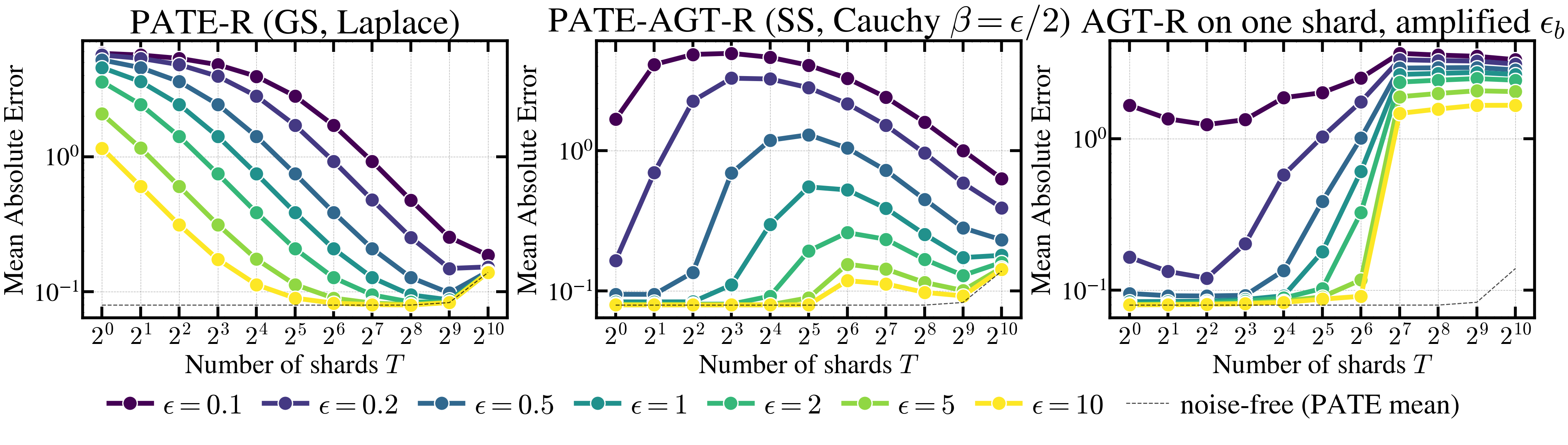}
    \caption{Shard ablation on the linear regression task with $4\!\times\!10^3$ training points, under pure DP. MAE against the number of shards $T$ at fixed dataset size, one line per budget $\epsilon \in [0.1, 10]$, for PATE-R (Laplace noise on global sensitivity), PATE-AGT-R (Cauchy noise on the shard-wise smooth sensitivity, $\beta = \epsilon/2$) and AGT-R on a single shard of $N/T$ points with the amplified budget $\epsilon_b$; dashed: the noise-free PATE mean.}
    \label{fig:pate_shard_ablation_linear}
\end{figure*}

In this section we validate the AGT-R certificate on its own. We first extend the analysis of the feasibility conditions of Figure~\ref{fig:cond_table}, then show the certificates behind the two ablations of the main text, the concretization frequency on linear regression and the batch size on California Housing. We then present a study that ablates the number of shards in PATE-R and PATE-AGT-R and reveals its effect on performance, and lastly repeat the comparison of Figure~\ref{fig:agtr_pate} on the raw error scale together with the global-sensitivity baselines, thus tying directly into our main experimental results.

\subsection{Condition Simulations: Further Analysis}\label{app:feasibility}

We complement the discussion of Figure~\ref{fig:cond_table} in the main text with a finer-grained analysis of the simulated conditions, starting with the pure $(\epsilon, 0)$-DP column. Additionally, we can observe that with high confidence (i.e. $> 90\%$), having a ratio $\rho \approx 0.01$, we can consistently achieve utility multipliers greater than $1$. A trend that additionally arises is that as $c$ increases, the gaps between consecutive $c$ decrease, meaning a smaller ratio consistently achieves $c \gg 1$ better utility.

We now turn to the finer details of the approximate $(\epsilon, \delta)$-DP column, where the admissible ratios peak at $\rho \approx 1.9$ for $c=1$. These largest values, however, are attained at $1 - \alpha \approx 0.8$, an operating point of little practical interest; at the more desirable $1 - \alpha = 0.98$ the admissible ratio falls to $1.12$--$1.44$, still favourable but far less generous.
A further important observation is that our analysis imposes conditions on $\beta$ and this constraint is stricter than in the pure-DP case: whereas Cauchy requires only $\beta < \epsilon/6$, the Laplace mechanism enforces $\beta \leq \frac{\epsilon}{2\ln(2/\delta)}$, roughly four times tighter at $\delta = 10^{-5}$. Additionally, as detailed in \S\ref{sec:agtr}, AGT-R produces strict upper-bounds on local sensitivity and is strictly over-approximate. Both these details force a practically larger $\mathrm{SS}^\beta$, or equivalently a larger $\epsilon$ to remain valid, so the displayed curves are optimistic ceilings only attainable theoretically. Lastly, the diminishing returns in $c$ reappear: $c=2$ and $c=3$ lie below unity across the entire window, peaking at $\rho \approx 0.96$ and $0.64$.

\subsection{Certificates of the AGT-R Ablations}\label{app:agtr_certificates}

Figure~\ref{fig:agtr_certificates} visualizes the $SS^\beta$ bounding algorithm behind the ablations of Figure~\ref{fig:agtr_ablation}, averaged over test inputs. The dashed line (left y-axis) represents the certified bound on the local sensitivity $\bar{A}(x,k)$ as the number of substitutions $k$ increases, while the solid line (right y-axis) tracks the induced upper bound on the smooth sensitivity, $\bar{A}(x,k)\,e^{-\beta k}$ at $\beta = 1/2$ (the $\beta$ of the pure-DP release at $\epsilon = 1$), with stars marking its maximum $SS^\beta$.
\begin{figure}[H]
  \centering
  \includegraphics[width=\columnwidth]{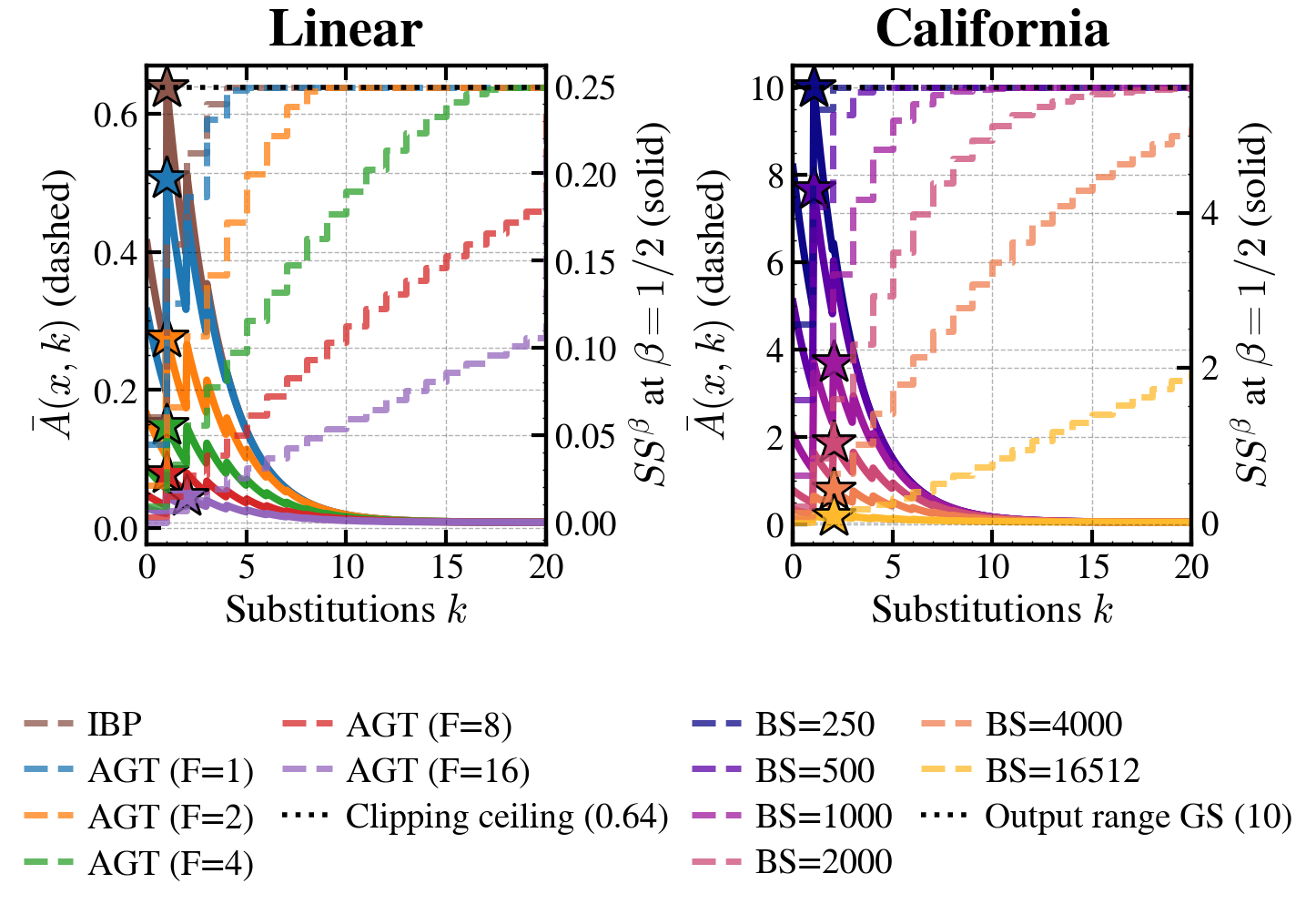}
  \caption{Cross-radius certificates of the AGT-R ablations (left: concretization frequency on linear regression; right: batch size on California Housing). Dashed, left axis: the certified local-sensitivity bound $\bar{A}(x,k)$ averaged over test inputs; solid, right axis: the smooth-sensitivity bound at $\beta = 1/2$, with its maximum starred; dotted: the global-sensitivity ceiling.}
  \label{fig:agtr_certificates}
\end{figure}
Beyond the largest computed $k$ the certificate falls back to the global sensitivity (dotted), which is the clipping ceiling of the trajectory ($0.64$) for linear regression and the output range ($10$ standard deviations) for California Housing; the released prediction is clipped to this range. The staircases show what the trade-off panels only show indirectly: a looser certificate (IBP, small $F$; small \texttt{BS}) saturates at the ceiling within a few substitutions, so the exponential discount has little to act on and $SS^\beta$ peaks at $k=1$ close to the ceiling, whereas a tight certificate (large $F$; full batch) stays well below the ceiling for all $k \leq 20$ and $SS^\beta$ peaks at $k = 2$ an order of magnitude lower. This is the mechanism behind the $0.25 \to 0.015$ and the $6 \to 0.08$ reductions in $SS^\beta$ reported in the main text.

\subsection{Number of shards ($T$)}

This ablation uses the Linear dataset and \textit{exactly} the same model and hyperparameters as in \S\ref{ssec:agtr_exp}, with the exception of the number of training points, which is now $4\!\times\!10^3$ instead of $4\!\times\!10^4$. Figure \ref{fig:pate_shard_ablation_linear} shows the MAE of PATE-R, PATE-AGT-R and AGT-R, the latter being restricted to one shard for $\epsilon$ ranging from $0.1$ to $10$, across an increasing number of shards, all while the dataset size is fixed. 
\begin{figure}[H]
  \centering
  \includegraphics[width=\columnwidth]{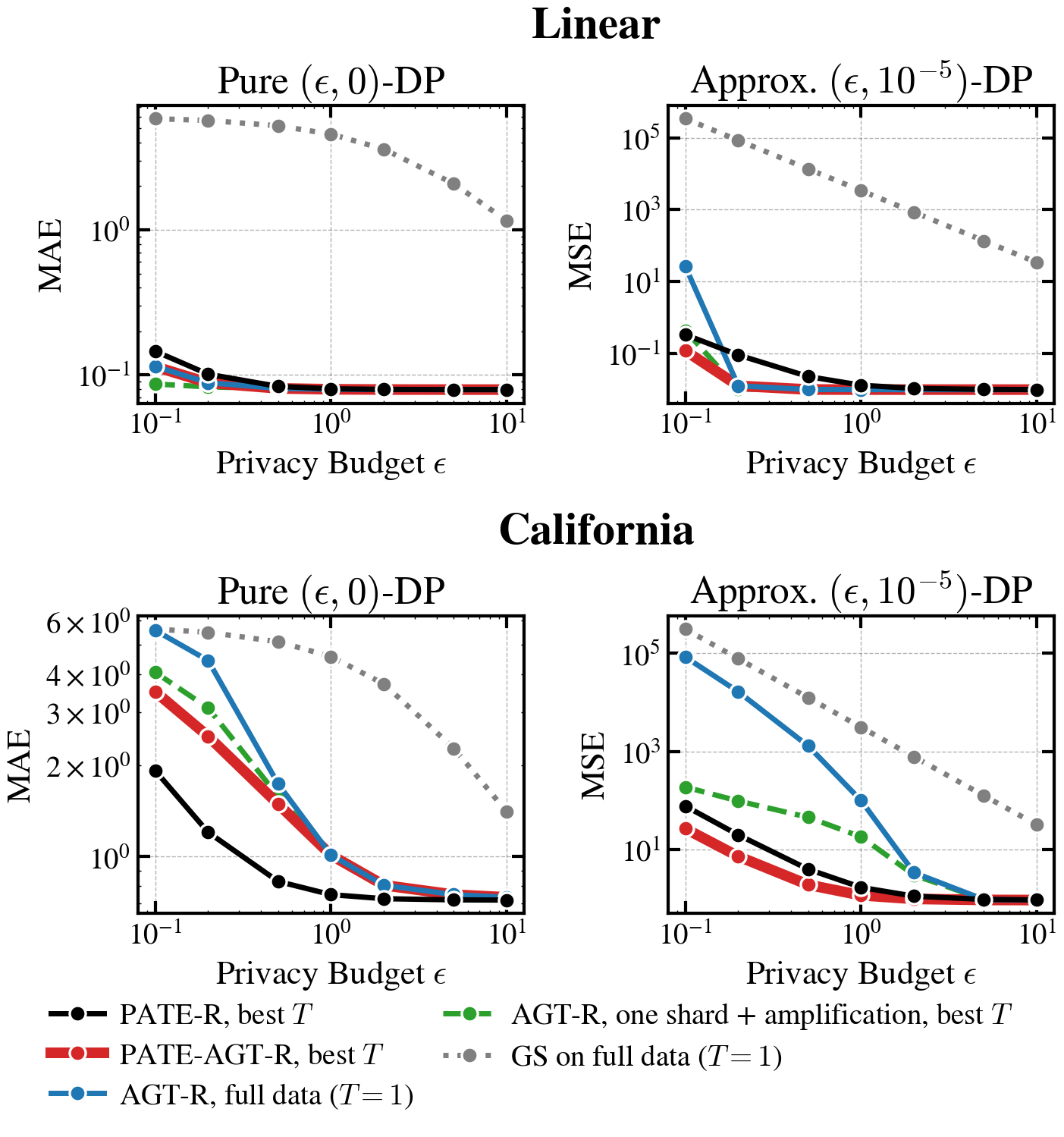}
  \caption{The comparison of Figure~\ref{fig:agtr_pate} on the raw scale with global sensitivity (grey, dotted). MAE under pure DP and test MSE under approximate DP ($\delta = 10^{-5}$), each point the mean over 200 noise draws.}
  \label{fig:agtr_gs}
\end{figure}

\subsection{Private Prediction Baselines with Global Sensitivity}\label{app:agtr_gs}
Figure~\ref{fig:agtr_gs} repeats Figure~\ref{fig:agtr_pate} on the raw error scale and with the global-sensitivity arm (Laplace in pure DP, Gaussian in approximate DP, both on the full data). On the excess scale of the main text this arm would sit one to five orders of magnitude above every other, which is why it is shown here. Accordingly, the y-axes report the full error of the private release, MAE under pure DP and MSE under approximate DP, rather than the excess over the non-private prediction; the non-private floor ($0.079$ MAE and $0.0098$ MSE on linear regression, $0.72$ MAE and $0.94$ MSE on California Housing) is where the AGT-based arms flatten out at large $\epsilon$, and subtracting it recovers Figure~\ref{fig:agtr_pate}.
\section{AGS Ablations}
\subsection{Sampling Ablation} \label{app:sampling_abl}
Intuitively, in standard PATE, MAE decreases as long as a single shard holds enough data for significant overlap without vote dilution, which also improves privacy by amplification. This is confirmed experimentally when looking at the leftmost subfigure and happens at $T=512$, i.e. $\approx 8$ points per shard, after which point the MAE starts increasing. The story changes somewhat in the case of PATE-AGT-R: the error starts to get larger earlier and the threshold at which utility starts decreasing is highly dependant on $\epsilon$. For example, MAE peaks at $T=8$ for $\epsilon=0.2$; in contrast the same peak is observed at $T=64$ when $\epsilon=2$. This indicates that under a certain specified privacy budget, noise calibrated to the parameter envelopes starts significantly affecting predictions. Whereas this budget is quickly reached as a function of the number of shards when the total privacy budget is already low, when this budget increases, the behaviour is observed at larger $T$. In the case of the first, the utility degradation is significant, while in the case of the latter it is not as visible due to the effect of mean aggregation. The rightmost panel of the figure can easily be interpreted: at $128$ shards, meaning $\approx 31$ datapoints per shard, AGT-R doesn't have enough data to compute tight parameter envelopes, which is why the test MAE skyrockets, regardless of total privacy budgets. Lower total privacy budgets start, expectedly, at a higher MAE due to noise causing utility degradation.

\begin{figure}[H]
 \centering
 \includegraphics[width=0.48\textwidth]{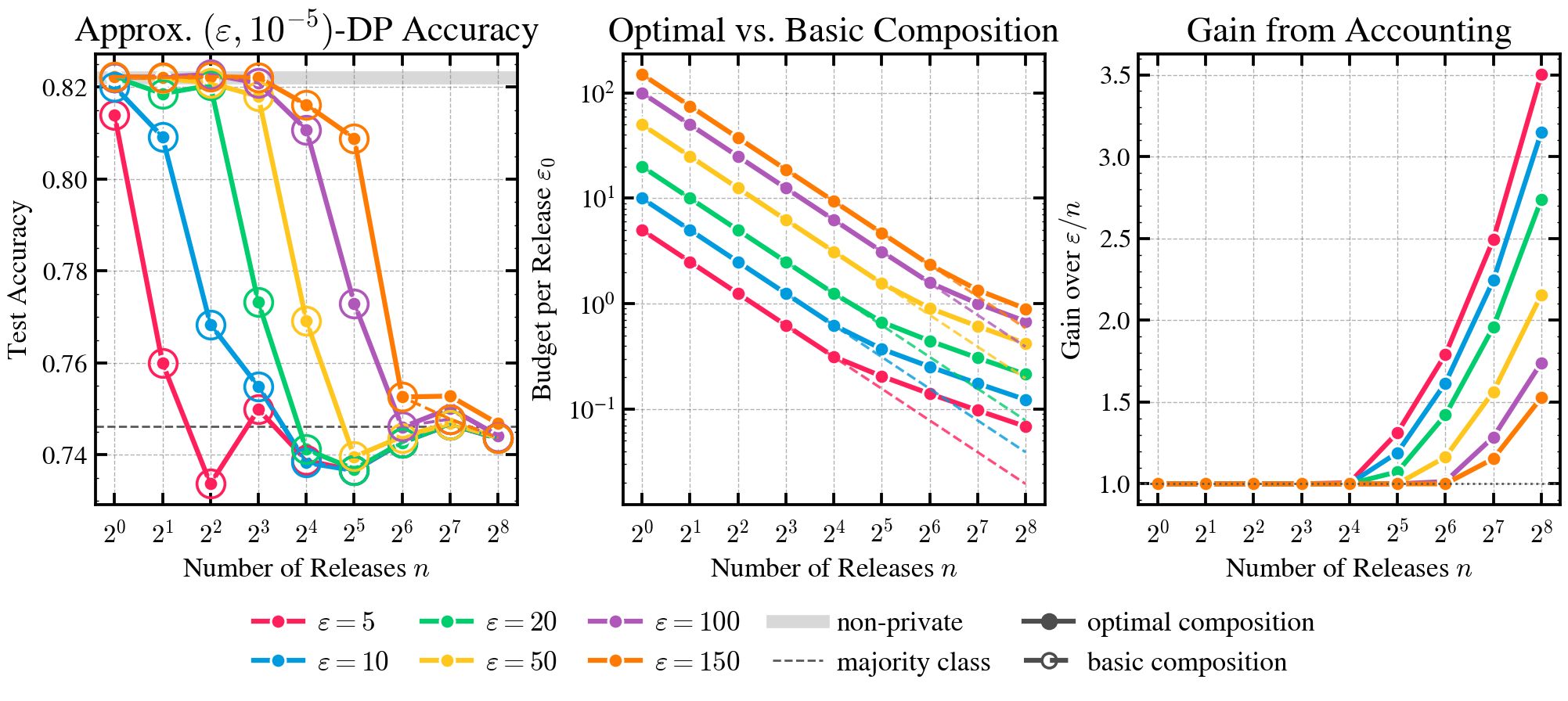}
 \caption{Ablation of the number of releases (i.e., sampling steps) at different, fixed total privacy budgets on the UCI Census Income Dataset \cite{adult1996}. \textbf{Left panel:} test accuracy against the number of releases $n$ under $(\epsilon, 10^{-5})$-DP, with the per-release budget set by optimal composition (filled) or basic composition (hollow), against the non-private accuracy (band) and the majority-class rate (dashed). \textbf{Middle:} the budget $\epsilon_0$ each release receives under optimal composition (solid) against the basic split $\epsilon/n$ (dashed). \textbf{Right:} the ratio of the two, i.e. the gain from accounting.}

 \label{fig:samplig_abl_adult}
\end{figure}

As previously mentioned in the beginning of \S\ref{sec:exp_ags}, we perform an ablation on the number of releases or `'sampling steps'` (i.e., the addition of noise and collapse of the parameter envelope) to observe the effect of composition on utility and privacy. We employ the UCI Census Income Dataset \cite{adult1996}, which has been used extensively to validate algorithmic techniques in DP literature \cite{chaudhuri2011differentially,dp_convex,dp_obj_pert_2023}.

The dataset consists of 30{,}222 private records. We train a logistic head on the first $d = 16$ PCA components with 300 steps of full-batch gradient descent split evenly into $n \in \{2^0, \dots, 2^8\}$ releases, each certified from the previous noisy release. We perform each ablation for a fixed total (approximate DP) privacy budget $\epsilon \in \{5,10,20,50,100,150\}$, $\delta=10^{-5}$, split across releases either evenly ($\epsilon/n$ each) or by optimal composition \cite{kairouz2015composition}.

Figure \ref{fig:samplig_abl_adult} shows the results. In the left panel, as the privacy budget increases, it takes more sampling steps for the test accuracy to decrease from $\approx 82\%$ to below $\approx 78\%$, which happens at $4$ releases for $\epsilon = 10$, at $16$ releases for $\epsilon = 50$ and at $64$ releases for $\epsilon = 150$, where $32$ releases still reach $80.9\%$. A larger total budget thus buys more free releases: the number of releases within one point of the non-private accuracy doubles from $8$ at $\epsilon = 50$ to $16$ at $\epsilon = 150$. By optimal composition, we only start gaining a larger $\epsilon$ per step (i.e., less noise injected) at the same cost $\delta$ when $n \geq 32$ (number of releases) for $\epsilon \leq 20$ and when $n \geq 2^7$ for $\epsilon \geq 100$, in which case the gains (up to $3.5\times$ at $\epsilon = 5$, $n = 2^8$) can be observed in the rightmost panel and the corresponding optimal $\epsilon$ in the middle panel. At $\epsilon \in \{100, 150\}$, this first shows in the accuracy ($75.0\%$ vs. $74.8\%$ and $75.3\%$ vs. $74.8\%$ at $n = 2^7$), although both remain near the majority-class rate ($74.6\%$).

\section{Hyperparameters}

\subsection{AGT-R} \label{app:hyperparam_agtr}

\begin{table}[H]
\centering
\scriptsize
\setlength{\tabcolsep}{3pt}
\renewcommand{\arraystretch}{1.08}
\begin{tabular}{@{}p{0.31\columnwidth}p{0.30\columnwidth}p{0.33\columnwidth}@{}}
\toprule
 & \textbf{Linear} & \textbf{California} \\
\midrule
\multicolumn{3}{@{}l}{\emph{Data}} \\
Train / test points & 40{,}000 / 2{,}000 & 16{,}512 / 4{,}128 \\
Output range (GS) & $[-6, 6]$ & $10$ std.\ units \\
\midrule
\multicolumn{3}{@{}l}{\emph{Baselines (Fig.~\ref{fig:agtr_pate})}} \\
Model & $y = wx + b$ & 8-64-1 ReLU \\
Steps / lr / $\gamma$ & 20 / 0.3 / 1 & 330 / 0.01 / 0.1 \\
Shards $T$ & $2^0, \dots, 2^{10}$ & $2^0, \dots, 2^6$ \\
Certificate & exact per step & AGT \cite{sosnin2025abstract} \\
Substitutions $k$ & $\le \min(1024, |\text{shard}|)$ & $\le 28$ \\
\midrule
\multicolumn{3}{@{}l}{\emph{Ablations (Fig.~\ref{fig:agtr_ablation})}} \\
Steps / lr / $\gamma$ & 16 (batch 5) / 0.1 / 0.1 & 330 / 0.01 / 0.1 \\
Ablated & $F \in \{\mathrm{IBP}, 1, \dots, 16\}$ & $\texttt{BS} \in \{250, \dots, 16{,}512\}$ \\
\midrule
\multicolumn{3}{@{}l}{\emph{Release (both datasets)}} \\
\multicolumn{3}{@{}l}{Pure DP: Cauchy$(\mathrm{SS}^\beta/(\epsilon-\beta))$, $\beta = \epsilon/2$; MAE of the clipped release} \\
\multicolumn{3}{@{}l}{Approx.\ DP: Laplace$(2\,\mathrm{SS}^\beta/\epsilon)$, $\beta = \epsilon/(2\ln(2/\delta))$, $\delta = 10^{-5}$; MSE} \\
\multicolumn{3}{@{}l}{$\epsilon \in \{0.1, 0.2, 0.5, 1, 2, 5, 10\}$; 200 noise draws per point} \\
\multicolumn{3}{@{}l}{WOR arm: $\epsilon_b = \ln(1 + T(e^{\epsilon} - 1))$, $\delta_b = T\delta$ on a shard of $N/T$} \\
\bottomrule
\end{tabular}
\caption{AGT-R settings for Figure~\ref{fig:agtr_real}. Training is full-batch (full-shard) clipped gradient descent. PATE-R adds Laplace$(\mathrm{GS}/(T\epsilon))$ (Gaussian in approximate DP) to the shard mean; PATE-AGT-R calibrates to $\frac{1}{T}\max_i \mathrm{SS}^\beta_i$.}
\label{tab:agtr_hyper}
\end{table}
\hfill

\subsection{AGS} \label{app:hyperparam_ags}
\begin{table}[H]
\centering
\scriptsize
\setlength{\tabcolsep}{3pt}
\renewcommand{\arraystretch}{1.08}
\begin{tabular}{@{}p{0.27\columnwidth}p{0.22\columnwidth}p{0.24\columnwidth}p{0.21\columnwidth}@{}}
\toprule
 & \textbf{Blobs} & \textbf{MNIST 0--4} & \textbf{IMDB} \\
\midrule
\multicolumn{4}{@{}l}{\emph{Data}} \\
Private / public / test & 200 / -- / 4k & 28.6k / 2k / 5.1k & 20k / 5k / 25k \\
Features & $\mathbb{R}^8$, 2 Gaussians & PCA-8 (pixels) & PCA-8 (mpnet) \\
\midrule
\multicolumn{4}{@{}l}{\emph{Model and training}} \\
Head ($P$) & linear (9) & logistic (45) & logistic (18) \\
Training & clipped SGD & \multicolumn{2}{l}{contractive clipped GD, full batch} \\
Steps / lr & 200 (b 20) / 0.5 & 1{,}000 / 2 & 1{,}000 / 2 \\
Clip $\gamma$ / contr.\ $\lambda$ & 0.05 / -- & 0.02 / 0.6 & 0.01 / 0.6 \\
\midrule
\multicolumn{4}{@{}l}{\emph{Certificate}} \\
Bound propagation & MILP, $F \in \{1, 2, 5, 10\}$ & CROWN & CROWN \\
Substitutions $k$ & $\{1, 2, 4, 8\}$ & $1, \dots, 4096, N$ & $1, \dots, 512, N$ \\
\midrule
\multicolumn{4}{@{}l}{\emph{Release (Corollary~\ref{thm:1_step_ags})}} \\
\multicolumn{4}{@{}l}{Pure DP: independent Cauchy$(2u/\epsilon)$, $\beta = \epsilon/(2P)$} \\
\multicolumn{4}{@{}l}{Approx.\ DP: independent Laplace$(2u/\epsilon)$, Gamma-tail $\beta$, $\delta = 10^{-5}$} \\
Budgets $\epsilon$ / draws & $1$--$100$ / 101 & $1$--$100$ / 41 & $1$--$100$ / 41 \\
Allocation & fixed ($\ell_1$) & \multicolumn{2}{l}{$\ell_1$ or anisotropic, chosen on public set} \\
\bottomrule
\end{tabular}
\caption{AGS settings for Figures~\ref{fig:blobs_ags} and~\ref{fig:combined_mnist_imdb}. $P$ is the number of released parameters and $u$ the smooth sensitivity of the $\ell_1$ cross-radius certificate. Public sets: MNIST 1{,}000 PATE pool and 1{,}000 selection; IMDB 2{,}500 and 2{,}500. Baselines are tuned on the same public sets.}
\label{tab:ags_hyper}
\end{table}

\begin{table}[H]
\centering
\scriptsize
\setlength{\tabcolsep}{3pt}
\renewcommand{\arraystretch}{1.08}
\begin{tabular}{@{}p{0.34\columnwidth}p{0.62\columnwidth}@{}}
\toprule
 & \textbf{SST-2 (hybrid)} \\
\midrule
\multicolumn{2}{@{}l}{\emph{Data}} \\
Private / public / test & 62{,}349 / 872 (GLUE validation) / 5{,}000 \\
Features & PCA-8 of the DP-LoRA all-mpnet-base-v2 embedding \\
\midrule
\multicolumn{2}{@{}l}{\emph{Encoder (DP-LoRA, $(\epsilon_1, 5\cdot 10^{-6})$-DP)}} \\
Adapter & rank 16 on the q, v projections of every layer \\
DP-SGD (Opacus) & batch 1{,}024, 3 epochs, clip 1, lr $10^{-3}$, PRV accountant \\
\midrule
\multicolumn{2}{@{}l}{\emph{Head (AGS, $(\epsilon_2, 5\cdot 10^{-6})$-DP)}} \\
Head ($P$) / training & logistic (18) / contractive clipped GD, full batch \\
Steps / lr / $\gamma$ / $\lambda$ & 150 / 2 / 0.05 / 0.3 \\
Certificate & CROWN; $k = 1, \dots, 512, N$ \\
Release & as Table~\ref{tab:ags_hyper}; 41 draws \\
\midrule
\multicolumn{2}{@{}l}{\emph{Pipeline}} \\
Total $\epsilon = \epsilon_1 + \epsilon_2$ & $0.5, 1, 2, 3, 5, 8$ at $\delta = 10^{-5}$ \\
Chosen on public set & the split $(\epsilon_1, \epsilon_2)$, the allocation, the threshold \\
\bottomrule
\end{tabular}
\caption{Settings of the SST-2 hybrid (Figure~\ref{fig:hybrid_ags_dpsgd}). The end-to-end arm is the DP-LoRA model scored with its own head; the frozen-encoder arms use the same head on the un-adapted embedding.}
\label{tab:sst2_hyper}
\end{table}

\end{document}